\documentclass[10pt,journal,compsoc]{IEEEtran}
\ifCLASSOPTIONcompsoc

\ifCLASSINFOpdf
\else
\fi

\usepackage{amsmath,amsfonts}
\usepackage{subfigure}
\usepackage{graphicx,graphics,color,psfrag}
\usepackage{algorithmic}
\usepackage{algorithm}
\usepackage{array}
\usepackage[caption=false,font=normalsize,labelfont=sf,textfont=sf]{subfig}
\usepackage{textcomp}
\usepackage{stfloats}
\usepackage{url}
\usepackage{verbatim}
\usepackage{cite}
\usepackage{amssymb}
\usepackage{mathtools}
\usepackage{amsthm}

\newtheorem{theorem}{Theorem}

\newtheorem{lemma}{Lemma}

\newtheorem{assumption}{Assumption}

\usepackage{multirow}
\usepackage{microtype}
\usepackage{booktabs} 
 
\begin{document}

\title{SCALA: Split Federated Learning with Concatenated Activations and Logit Adjustments}

\author{Jiarong Yang and Yuan Liu
\IEEEcompsocitemizethanks{\IEEEcompsocthanksitem Jiarong Yang and Yuan Liu are with school of Electronic and Information Engineering, South China University of Technology, Guangzhou 510641, China (e-mails: eejryang@mail.scut.edu.cn, eeyliu@scut.edu.cn). Corresponding author: Yuan Liu.
}
}

\markboth{Journal of \LaTeX\ Class Files,~Vol.~14, No.~8, August~2015}%
{Shell \MakeLowercase{\textit{et al.}}: Bare Demo of IEEEtran.cls for Computer Society Journals}

\IEEEtitleabstractindextext{%
\begin{abstract}
Split Federated Learning (SFL) is a distributed machine learning framework where the models are split and trained on the server and clients. However, data heterogeneity and partial client participation result in label distribution skew, which severely degrades learning performance.
To address this issue, we propose SFL with Concatenated Activations and Logit Adjustments (SCALA), in which activations from client-side models are concatenated as the input of the server-side model to centrally adjust label distribution across different clients, and logit adjustments in the loss functions of both server-side and client-side models are performed to deal with the label distribution variation across different subsets of participating clients. Theoretical analysis demonstrates that the concatenation of activations reduces the impact of data heterogeneity, while logit adjustments in loss functions enhance the recognition of low-frequency labels at the cost of sacrificing the recognition of high-frequency labels.
\end{abstract}

}

\maketitle



\IEEEraisesectionheading{\section{Introduction}}
\IEEEPARstart{S}{plit} Federated Learning (SFL) \cite{9016486,thapa2022splitfed,gawali2021comparison} is a promising technology that integrates the characteristics of both Federated Learning (FL) \cite{9084352,mcmahan2017communication} and Split Learning (SL) \cite{gupta2018distributed}, where model splitting and aggregation are concurrently deployed, as depicted in Fig. \ref{SFLvsCSFL}(a). Specifically, akin to SL, SFL also splits the artificial intelligence (AI) model into two parts, with each part deployed separately on the server and clients. In the training phase, each client processes local data through the client-side model, and then sends the activations to the server to complete the forward propagation. Then, the server updates the server-side model and sends back the backpropagated gradients to each client to finish the update of the client-side model. In this way, clients only need to execute the initial layers of complex AI models, such as deep neural networks, and offload the remaining layers to the server for execution. On the other hand, SFL employs parallel training to enhance training efficiency and utilizes FL technology to aggregate server-side and client-side models. 
\begin{figure}[t]
\centering
\includegraphics[width=8.5cm]{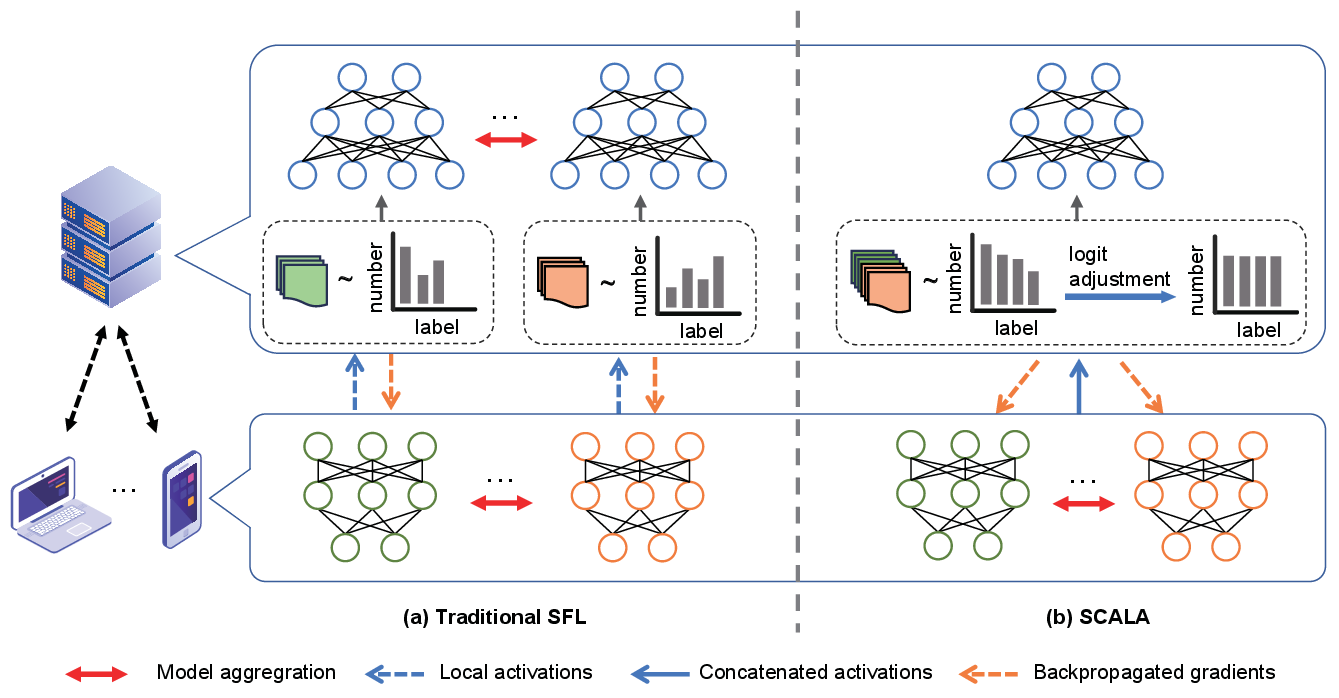}
\caption{An illustration of traditional SFL and SCALA in scenarios with skewed label distribution. Traditional SFL maintains and trains server-side models for each participating client on the server and periodically aggregates these server-side models. SCALA maintains and trains one server-side model based on concatenated activations.}
\label{SFLvsCSFL}
\end{figure}
\par
Same as other distributed machine learning methods, SFL also encounters the issue of skewed label distribution due to two main factors. The first is local label distribution skew \cite{zhang2022federated,wang2024aggregation,10510418} caused by the data heterogeneity of clients where label distribution varies across clients. The second is  global label distribution skew \cite{zhang2023fed,10910050} brought by partial client participation, resulting in label distribution variation across different subsets of participating clients. Specifically, the local objective functions are inconsistent among clients due to local label distribution skew, causing the aggregated model to deviate from the global optimum \cite{wang2020tackling,acar2021federated,wang2024aggregation}. On the other hand, partial client participation \cite{yang2021achieving} is usually assumed, where only a subset of clients are involved in the model training due to heterogenous communication and computational capabilities of clients. However, in this case the server can not access label distribution of non-participating clients, which is likely to exhibit a skewed label distribution and fails to learn a global model that can generalize well for all clients \cite{kim2024communication}. 
\par
The common approaches to skewed label distribution are regularization \cite{li2020federated,shi2023towards,kim2024communication}, loss function calibration \cite{lee2022preservation,zhang2022federated}, and client selection \cite{zhang2023fed,10910050}. However, if the local label distribution is highly skewed, there are missing classes in the local data and the local model still exhibits adverse bias towards the existing labels. Moreover, the adverse bias is present in every layer of the model and intensifies as the depth of the model increases, and thus the classifier is the most severely affected \cite{NEURIPS2021_2f2b2656}. This issue is unavoidable in the distributed machine learning algorithms, where each client may lack certain classes of its local data. 
\par
To address the above issue of skewed label distribution, we propose SCALA whose core idea is to centrally adjust label information from clients at the server. Specifically, to deal with the local label distribution skew raised from data heterogeneity, as shown in Fig. \ref{SFLvsCSFL}(b), we concatenate the activations (output of the client-side models) to serve as the input for the server-side model (i.e., the parallel SL step). This enables the server to train the server-side model centrally based on a dataset with concatenated label distribution, thereby effectively alleviating the bias introduced by local label distribution skew with missing classes. Concurrently, the clients share the same server-side model during training, directly addressing the deep-layer adverse bias as demonstrated in prior studies \cite{NEURIPS2021_2f2b2656,shang2022federated}. 
\par
Note that although recent works, such as \cite{huang2023minibatch,liao2024mergesfl}, also centrally train server-side model, they ignore the global label distribution skew. When few clients participate, e.g., in extreme cases with only two clients, the server cannot achieve a relatively balanced activations through concatenation. In such cases, the label distribution of the concatenated activations exhibit a long-tailed distribution \cite{menon2021longtail,Zhong_2021_CVPR,10105457}, resulting in common loss functions being less applicable.
Specifically, the misclassification error of loss functions is related to label distribution, where high-frequency labels result in a smaller misclassification error, leading the model to focus on improving predictive accuracy for high-frequency labels.
To tackle the global label distribution skew due to partial client participation, we perform logit adjustments for both server-side and client-side models to mitigate the impact of label distribution on loss function updates. In detail, we employ logit adjustment to decrease the logits for high-frequency labels by a larger value and for low-frequency labels by a smaller value, thereby achieving a calibrated misclassification error that is independent of label distribution. Our main contributions are summarized as follows.
\begin{itemize}
    \item We propose  SCALA (Split Federated Learning with Concatenated Activations and Logit Adjustments) to address label distribution skew at both local and global levels. The proposed SCALA mitigates the impact of missing classes in highly skewed local label distributions by centrally training the server-side model using concatenated activations from clients. Additionally, logits before softmax cross-entropy are adjusted according to the concatenated label distribution to address global label distribution skew.
    \item We perform a convergence analysis of the concatenated activations enabled SFL framework and derive the following insights: First, increasing the number of server-side model layers reduces the gradient dissimilarity, thereby mitigating the impact of local label distribution skew. Second, increasing the client participation rate alleviates the influence of global label distribution skew on convergence performance.
    We theoretically derive that logit adjustments enhance low-frequency label recognition at the cost of high-frequency accuracy. Furthermore, our experiments reveal that the efficacy of logit adjustments is fundamentally predicated on concatenated activations, where the holistic view provided by concatenation is essential for effective calibration.
\end{itemize}

\section{Related Works}

\begin{table*}[ht]
\caption{Comparison among different FL and SFL Algorithms Addressing Label Distribution Skew.}
\label{Comparison2}
\begin{tabular}{cccll}
\toprule
\textbf{Algorithms} & \textbf{Type} & \textbf{Skew Addressed} & \multicolumn{1}{c}{\textbf{Main Advantage}}  & \multicolumn{1}{c}{\textbf{Principal Limitation}} \\ \midrule
\textbf{FedACG} \cite{kim2024communication}, & FL            & Local                   & Introduce penalties to the local loss,       & Performance degrades when clients                 \\
\textbf{FedProx} \cite{li2020federated},     &               &                         & constraining the divergence between          & lack certain labels.                     \\
\textbf{FedDecorr} \cite{shi2023towards}     &               &                         & local and global objectives under            &                                                   \\
                                    &               &                         & heterogeneous data.                          &                                                   \\ \midrule
\textbf{FedLC} \cite{zhang2022federated},    & FL            & Local                   & Adjust the local loss according to           & Performance degrades when clients                 \\
\textbf{FedNTD} \cite{lee2022preservation}   &               &                         & client-specific class statistics, yielding   & lack certain labels.                     \\
                                    &               &                         & higher-quality local updates.                &                                                   \\\midrule
\textbf{CCVR} \cite{NEURIPS2021_2f2b2656},   & FL            & Local                   & Retrain the global classifier with virtual   & Requires accurate generation and storage          \\
\textbf{CReFF} \cite{shang2022federated}     &               &                         & prototypes, rebalancing the decision         & of virtual features; does not correct bias        \\
                                    &               &                         & boundary from a centralized perspective.     & in feature extraction layers.                     \\ \midrule
\textbf{FedConcat} \cite{diao2024exploiting} & FL            & Local                   & Concatenates the outputs of feature          & Introduces additional fine-tuning steps;          \\
                                    &               &                         & extraction layers instead of averaging       &  results in a larger global model with            \\
                                    &               &                         & their parameters, preserving information     &  higher storage and inference costs.              \\
                                    &               &                         & from rare or missing classes.                &                                                   \\\midrule
\textbf{FedCBS} \cite{zhang2023fed},         & FL            & Global                  & Dynamically selects a class-balanced         & Incurs coordination overhead; scheduling          \\
\textbf{AdaFL} \cite{li2024adafl}            &               &                         & subset of participants each round,           & efficacy is vulnerable to fluctuating             \\
                                    &               &                         & alleviating global label distribution skew.  & network conditions.                               \\ \midrule
\textbf{Minibatch-SFL} \cite{huang2023minibatch},  & SFL           & Local                   & Centrally trains the server-side model on    & Performance degrades when few clients         \\
\textbf{MergeSFL} \cite{liao2024mergesfl}    &               &                         & concatenated activations, preserving         & participate.                                      \\
                                    &               &                         & information from rare or missing classes.     &                                                   \\ \midrule
\textbf{CS-SFL} \cite{10910050}               & SFL           & Global                  & Integrates label-aware client sampling into       & Incurs coordination overhead; scheduling          \\
                                    &               &                         & SFL, mitigating global label distribution & efficacy is vulnerable to fluctuating             \\
                                    &               &                         & skew.                                & network conditions.                               \\ \midrule
\textbf{SCALA} (ours)                               & SFL           & Local and global        & Address both local and global label              & Lack of communication optimization                                                  \\
                                    &               &                         & distribution skew through concatenated       & mechanisms.                         \\
                                    &               &                         & activations and logit adjustment.          &                                                    \\  \bottomrule  
\end{tabular}
\end{table*}

\subsection{Federated Learning}
FL emerges from Google keyboard personalization pipeline and is first formalized by \cite{mcmahan2017communication} as the FedAvg algorithm, which performs iterative model averaging across a population of resource-constrained clients without exporting raw data \cite{9084352,mcmahan2017communication}. This decentralized paradigm rapidly gains traction because it simultaneously satisfies modern privacy regulations (e.g., GDPR \cite{voigt2017eu}, CCPA \cite{bukaty2019california}) and alleviates the burden of aggregating large-scale, siloed datasets \cite{kairouz2021advances}. Subsequent studies have extended the applicability of FL to complex distributed environments by (1) addressing statistical heterogeneity across clients and (2)  improving communication and computation efficiency. 
\par
The first research direction addresses statistical heterogeneity among client datasets, which is inherent in decentralized settings where local data are often non-independent and identically distributed (non-IID) \cite{zhu2021federated,zhang2021federated,ma2022state}. Such heterogeneity manifests in multiple forms, including label distribution skew \cite{zhang2022federated,wang2024aggregation,zhang2023fed,10910050,10510418}, feature skew \cite{zhou2023fedfa,zhou2024federated,yan2025simple}, and quantity skew \cite{karimireddy2020scaffold,wang2020tackling}, each of which distorts gradient alignment across clients and hampers global model convergence. Several representative surveys have revealed that label distribution skew has the most pronounced effect on training instability \cite{hsieh2020non,9084352,kairouz2021advances}. Specifically, label distribution skew can be categorized into local label distribution skew and global label distribution skew. Local label distribution skew is caused by the data heterogeneity of clients where the local label distribution of each client deviates from the overall population distribution. Such misalignment perturbs local gradient directions and increases the variance of model updates \cite{lee2022preservation,zhang2022federated}. Global label distribution skew stems from the practical constraint of partial client participation. Since many clients possess limited compute, battery life, or network connectivity, it is infeasible for every client to take part in each communication round \cite{yang2020federated,zhu2020toward}. Consequently, only a subset of clients participate in model tarining per round, and this participation is inherently unbalanced: resource-rich clients participate far more frequently than resource-constrained ones. The label histogram observed at the server therefore mirrors the participation bias—often manifesting a long-tailed class distribution—rather than the true population distribution, leading to a persistent divergence in the model updates \cite{zhang2023fed,li2024adafl,10910050}. 
\par
Numerous methods have been proposed to address this issue, with recent studies in FL tackling local label distribution skew through optimizations on both the client and server sides. Specifically, \cite{li2020federated,tan2022fedproto,shi2023towards,mu2023fedproc,kim2024communication} add extra regularization terms to the local loss function to reduce bias between local and global models. \cite{lee2022preservation,zhang2022federated} calibrate the local loss function based on the characteristics of local data to produce higher quality local models. And FedConcat directly combines local models via concatenation instead of averaging to address label-skewed local data \cite{diao2024exploiting}. Additionally, global classifiers are calibrated at the server side using virtual representations in \cite{NEURIPS2021_2f2b2656,shang2022federated}. Although these approaches improve robustness under local label distribution skew, they fundamentally rely on local data for model updates, and their effectiveness deteriorates when clients entirely lack samples from certain classes. To address global label distribution skew, recent works in FL employ client selection to choose a balanced subset of participating clients to address this issue \cite{zhang2023fed,li2024adafl,10910050,yang2023client}. However, these methods often incur non-negligible coordination and communication overhead and are sensitive to dynamic network conditions.
\par
The second research direction focuses on reducing the significant communication overhead that arises from iterative model synchronization between clients and the server \cite{niknam2020federated,yang2022federated,10759679}. Techniques such as gradient sparsification, adaptive client scheduling, and asynchronous update schemes have been extensively explored to mitigate bandwidth consumption while maintaining model accuracy. For example, gradient sparsification selectively transmits only the most significant gradient elements and adaptively adjusts the sparsity level to balance communication efficiency and learning performance \cite{han2020adaptive,lin2023joint,vaishnav2024communication}. Adaptive client scheduling strategies further optimize participation frequency based on client availability and network reliability, achieving more efficient use of communication resources \cite{xu2021online,cui2021client,10197174,yang2023client,11215609}. Asynchronous update schemes aim to optimize model performance under asynchronous client participation through staleness-aware model aggregation or client selection, thereby mitigating the adverse impact of stragglers on training efficiency \cite{wang2022asynchronous,xu2023asynchronous,10346989}.
\par
Despite these advances, communication and computation costs remain a dominant bottleneck in large-scale federated networks, particularly when high-dimensional deep neural models are trained over unstable wireless or edge infrastructures \cite{luo2021cost,xu2023accelerating}. Specifically, traditional FL assumes that clients can execute full forward–backward propagation locally, which may be infeasible for clients with constrained computation or memory. Moreover, as the model depth increases, the uplink communication cost grows proportionally with the number of parameters transmitted per round, imposing excessive strain on bandwidth-limited clients. As FL applications extend from lightweight models to more complex architectures such as deep neural networks and Transformers \cite{vaswani2017attention}, these issues become increasingly pronounced.
\subsection{Split Learning}
SL addresses the compute bottleneck by splitting the model at a cut-layer \cite{vepakomma2018split,11176835}. In a typical SL configuration, each participating client retains the front portion of the model, comprising the initial convolutional or feature-extraction layers, while the server hosts the remaining deeper layers responsible for higher-level representation learning and classification. During training, a client performs a partial forward propagation on its local data up to the cut layer and transmits the resulting intermediate activations (often referred to as “smashed data”) to the server. The server continues the forward pass through its portion of the model to compute the final output and evaluate the loss function using ground-truth labels. Subsequently, the server initiates backward propagation by calculating gradients with respect to the received activations and sends these gradients back to the client. The client then completes the backward pass through its local layers to update its parameters. This client–server coordination completes one round of training, and the process is repeated until convergence. 
\par 
While SL markedly reduces the computational burden on clients, its inherently sequential workflow, requiring each client to wait for the server to complete its computation before proceeding, severely restricts overall training efficiency, thereby constraining its applicability in large-scale collaborative scenarios \cite{thapa2022splitfed}.

\subsection{Split Federated Learning}
To overcome the limitations of the sequential update process in SL, researchers have proposed SFL, a hybrid paradigm that integrates the parallel model aggregation mechanism of FL with the model-splitting framework of SL \cite{turina2021federated}. In traditional SFL frameworks, such as SplitFed, clients perform local training in parallel, and their model updates are subsequently aggregated in the same manner as FL to update the global client-side model \cite{thapa2022splitfed}. Specifically, the SFL training process unfolds in three main stages within each global communication round. First, each client executes forward propagation up to the cut layer on its local data and transmits intermediate activations to the server. Second, the server performs forward and backward propagation on its portion of the model using the received activations, computes gradients, and sends gradients back to each client. Finally, clients complete the backward pass and update their local client-side models. Once local training epochs are finished, all client-side models are aggregated by the central server, following the FedAvg principle. The updated global client-side model is then redistributed to clients for the next communication round. This architecture effectively alleviates the computational bottleneck of traditional FL, as each client is responsible for training only a portion of the model \cite{singh2019detailed}. Meanwhile, by enabling parallel client participation, it mitigates the sequential dependency inherent in SL. Furthermore, by aggregating gradient updates from multiple clients, SFL achieves better generalization performance than pure SL systems, which typically train each client–server pair independently \cite{yang2024gas}. Consequently, SFL has attracted increasing attention across domains that demand both privacy and efficiency, such as federated medical imaging, edge-based anomaly detection, and intelligent transportation systems \cite{yang2022robust}. 
\par
Subsequent studies have advanced SFL along three main directions: (1) further reducing communication costs via activation quantization and pruning; (2) enhancing security through differential privacy and inversion-aware regularization; and (3) mitigating the heterogeneity of compute and bandwidth by leveraging generative activations, elastic partitioning, or ring-based scheduling \cite{yang2024gas,10522472}. Specifically, the first direction focuses on reducing communication costs between clients and the server, which remain significant due to the frequent transmission of high-dimensional activations and their corresponding gradients \cite{wu2023communication}. To alleviate this burden, researchers have explored activation quantization and adaptive pruning techniques that compress intermediate activations without substantially degrading model accuracy. For instance, FedLite leverages product quantization combined with a gradient correction mechanism to effectively compress activations and restore accurate gradients \cite{wang2022fedlite}. 
\par
The second research direction focuses on enhancing security and privacy preservation, addressing the potential vulnerability of SFL to activation inversion attacks. On one hand, researchers employ differential privacy mechanisms that inject carefully calibrated random noise into activations before transmission, effectively obfuscating individual data contributions and preventing the server or other participants from inferring sensitive information \cite{thapa2022splitfed}. On the other hand, inversion-aware adversarial objectives are introduced to improve robustness against model inversion attacks by jointly optimizing two components: (1) a strong inversion model is jointly trained to mimic the attacker’s reconstruction process, and its reconstruction quality is used as a regularization term to discourage easily invertible representations \cite{li2022ressfl,vepakomma2020nopeek}; and (2) bottleneck layers are designed to compress the intermediate feature space, thereby reducing information leakage while maintaining high model accuracy \cite{li2022ressfl}.
\par
The third research direction aims to address client heterogeneity in computation and bandwidth resources. To this end, researchers have proposed adaptive frameworks that employ elastic model-splitting strategies, dynamically adjusting the cut layer according to client hardware configurations \cite{xu2023accelerating,10522472}. In addition, ring-based scheduling mechanisms and asynchronous update schemes are introduced to accommodate unbalanced client participation \cite{shen2023ringsfl}. More recently, generative activations have been introduced to mitigate activation drift caused by asynchronously participating clients \cite{yang2024gas}.
\par
Despite these advances, research on label distribution skew in SFL is still in its infancy: Minibatch-SFL \cite{huang2023minibatch} and MergeSFL \cite{liao2024mergesfl} mitigate the impact of local label distribution skew by centrally training the server-side model and recombining activations, while CS-SFL introduces class-balanced client selection to address global label distribution skew in wireless networks \cite{10910050}.

\subsection{Summary and Motivation for SCALA}
The above discussion highlights the evolutionary progression from FL to SL, and ultimately to SFL. FL provides a decentralized training paradigm that enables large-scale collaboration while preserving client data privacy. However, its conventional architecture assumes that each client possesses sufficient computational resources to train a complete model locally and has adequate uplink bandwidth to frequently exchange high-dimensional parameters with the server. In practice, these assumptions rarely hold, particularly when deploying deep convolutional or transformer-based networks on mobile or IoT devices, leading to performance degradation and communication inefficiency. To alleviate the computational burden on clients, SL decomposes the model into client-side and server-side segments, enabling lightweight on-client computation while delegating deeper model layers to a centralized server. Although SL achieves substantial efficiency gains, it introduces a strict sequential dependency between clients and the server, making the training process inherently non-parallel and thus inefficient in multi-client settings. To integrate the advantages of both paradigms, SFL has been proposed as a hybrid framework. SFL allows multiple clients to perform split training in parallel and subsequently aggregates their client-side models in a federated manner. By enabling parallel training, SFL effectively overcomes the sequential bottleneck of SL while reducing the computational load on clients compared with FL. 
\par
Nevertheless, research on label distribution skew within the SFL framework remains in its infancy. Table \ref{Comparison2} summarizes the related algorithms, from which it can be observed that existing studies lack lightweight strategies capable of simultaneously addressing both local and global label distribution skew. To bridge this gap, we proposes SCALA, which enhances the capability of SFL in handling label distribution skew. SCALA trains the server-side model using concatenated activations, thereby mitigating the inherent inconsistency between local and global models observed in FL frameworks \cite{li2020federated,lee2022preservation,zhang2022federated}, particularly when local clients lack samples from certain classes. In addition, unlike previous works \cite{zhang2023fed,li2024adafl,10910050} that primarily rely on client-selection strategies to address global label imbalance, SCALA introduces a novel logit adjustment mechanism that directly calibrates the loss function, effectively alleviating the adverse effects of long-tailed label distributions.
Note that SCALA relies on frequent activation exchange similar to traditional SFL, leading to high communication costs. Since communication optimization is beyond the scope of this paper, designing efficient variants (e.g., via activation compression) remains a valuable direction for future research.

\section{Proposed Method: SCALA}
\begin{figure}[t]
\centering
 \includegraphics[width=8cm]{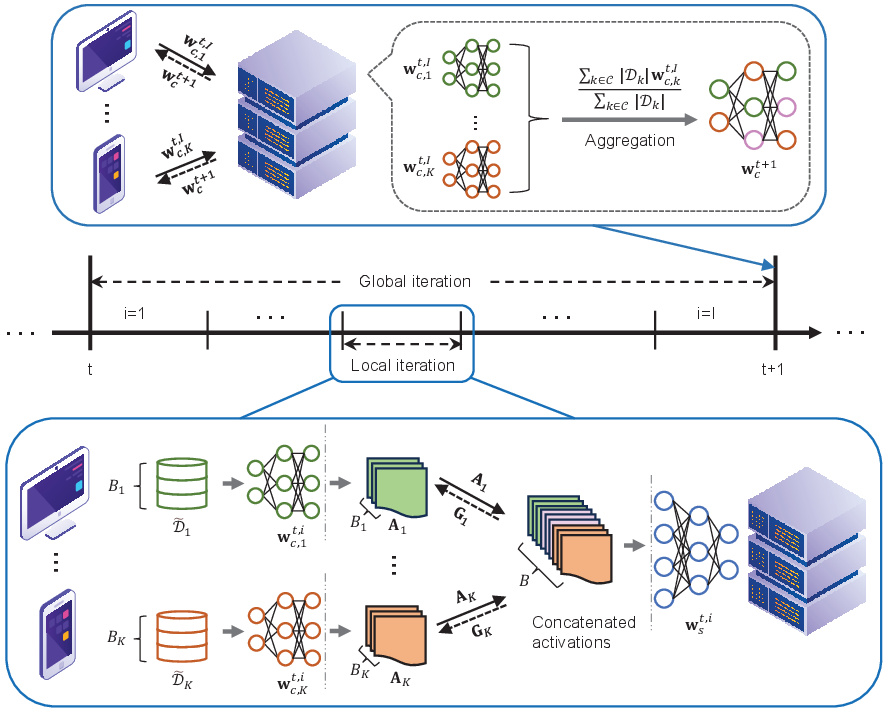}
\caption{Concatenated activations enabled SFL framework. All participating clients synchronously execute local iterations, where the client-side models are updated locally and sent to the server for aggregation at the $I$-th local iteration. The server-side model is centrally updated in each local iteration, where the input is concatenated from the activations uploaded by participating clients. }
\label{systemmodel}
\end{figure}

In this section, we first present the motivations behind the two core modules in SCALA, namely concatenated activations and logit adjustments. We then describe the concatenated activations enabled SFL framework, followed by proposing the loss functions with logit adjustments for the server-side and client-side models.


\subsection{Motivations Behind the Two Proposed Modules in SCALA}

\textbf{Addressing Local Label Distribution Skew through Concatenated Activations.} The concatenation of activations tackles the issue of missing classes caused by local label distribution skew. When local label distributions are highly skewed, clients tend to learn biased local models, and this bias amplifies with depth, most severely distorting the classifier \cite{NEURIPS2021_2f2b2656}. By concatenating the activations uploaded by participating clients and training the deeper layers centrally, the server effectively forms a mini-batch whose class support is the union of the participating clients’ classes. Each client thus contributes complementary feature representations,  which mitigates the adverse effect of isolated or missing classes.
\par
\textbf{Handling Global Label Distribution Skew via Logit Adjustments.} Despite resolving the challenge of local label distribution skew, the concatenated activations introduce a global label distribution skew when only a subset of clients participates in training. With partial participation, the class histogram observed at the server in any round reflects participation bias rather than the true population; even after concatenation it can be long-tailed, as shown in Fig. \ref{concat}. Standard cross-entropy is implicitly frequency-aware, where the expected misclassification penalty for a class scales with its prior, causing the model to overfit frequent classes and under-train rare ones. To counteract this, SCALA employs logit adjustments in the loss functions for both server-side and client-side models, ensuring balanced training that accommodates the skewed label distribution and enhances predictive performance across all labels.

\subsection{Concatenated Activations Enabled SFL Framework}
\subsubsection{Preliminaries}
We consider $K$ clients participating in the training indexed by $\mathcal{K}=\{1,2,\cdots,K\}$, each holding local data $\mathcal{D}_k$ of size $|\mathcal{D}_k|$. The goal of the clients is to learn a global model $\mathbf{w}$ with the help of the server. In SFL, the global model $\mathbf{w}$ with $N$ layers is split into the client-side model $\mathbf{w}_c$ with $N_c$ layers, which is collaboratively trained and updated by the clients and the server, and the server-side model $\mathbf{w}_s$ with $N_s$ layers, which is stored and trained solely on the server. The global model $\mathbf{w}$ is obtained by minimizing the averaged loss over all participating clients $\mathcal{C}$ as
\begin{align}
\min_{\mathbf{w}}F(\mathbf{w}):=\frac{\sum_{k\in \mathcal{C}}|\mathcal{D}_k|F_k(\mathbf{w})}{\sum_{k\in \mathcal{C}}|\mathcal{D}_k|},
\end{align}
where $F_k(\mathbf{w})$ denotes the local expected loss function, and it is unbiasedly  estimated by the empirical loss $f_k(\cdot)$, i.e., $\mathbb{E}_{\tilde{\mathcal{D}}_k\thicksim\mathcal{D}_k}f_k(\mathbf{w},\tilde{\mathcal{D}}_k)=F_k(\mathbf{w})$. In addition, under the setting of model splitting, the empirical loss of client $k$ is formulated as
\begin{align}
    f_k(\mathbf{w}_c) = l(\mathbf{w}_{s};h(\mathbf{w}_{c}; \tilde{\mathcal{D}}_k)),
\end{align}
where $h$ is the client-side function mapping the input data to the activation space and $l$ is the server-side function mapping the activation to a scalar loss value.
\subsubsection{Algorithm Description}
As shown in Fig. \ref{systemmodel}, we propose a concatenated activations enabled SFL framework. The server-side model is updated based on the concatenated activations, which we refer to as the parallel SL phase. And the client-side model is updated through aggregation, which we refer to as the FL phase. 
At the start of training, the server will set the minibatch size $B$, the number of local iterations $I$ indexed by $i$, and the number of global iterations $T$ indexed by $t$. Given a set of clients $\mathcal{K}=\{1,2,\cdots,K\}$, concatenated activations enabled SFL will output a global model $\mathbf{w}^T=[\mathbf{w}_s^T, \mathbf{w}_c^T]$ after $T$ global iterations. Taking the $t$-th global iteration as an example, we provide a detailed description of the training process in the following, where the superscript $t$ is omitted for notational brevity.
\par
\textbf{Client-side models deployments.} At the beginning of a global iteration, the server randomly selects a subset of clients $\mathcal{C}$ and sets the minibatch size for each client $k\in \mathcal{C}$ in proportion to its local data size $|\mathcal{D}_k|$ as
\begin{align}\label{eqn:bk}
    B_k = \frac{|\mathcal{D}_k|B}{\sum_{k\in \mathcal{C}} |\mathcal{D}_k|}, 
\end{align}
so that the number of local iterations can be synchronized.
Then each client $k\in \mathcal{C}$ downloads the client-side model $\mathbf{w}_c$ and the minibatch size $B_k$ from the server. 
\par
\begin{figure}[t]
\centering
 \includegraphics[width=6cm]{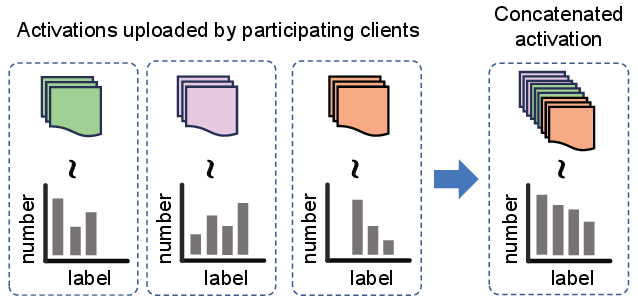}
\caption{The process of concatenating activations. The activations uploaded by participating clients are concatenated to serve as the input for the server-side model, effectively mitigating the issue of missing classes under a highly skewed local label distribution.}
\label{concat}
\end{figure}
\textbf{Parallel SL phase.} The selected clients and the server collaborate to perform parallel SL over $I$ local iterations, during which the server-side model is updated at each local iteration. It is worth noting that the server-side model is updated at each local iteration in SCALA, whereas that is updated at each global iteration, i.e., at the $I$-th local iteration, in traditional SFL. The detailed description of the process is as follows:
\begin{itemize}
\item \textbf{Step 1 (Forward propagation of client-side models):} Each client $k\in \mathcal{C}$ randomly selects a minibatch $\tilde{\mathcal{D}}_k \in \mathcal{D}_k $ of size $B_k$ as  $\tilde{\mathcal{D}}_k = \{(\mathbf{x}_1,y_1),(\mathbf{x}_2,y_2),\cdots,(\mathbf{x}_{B_k},y_{B_k})\}$, where the sample set of the minibatch is denoted by $\mathbf{X}_k=\{\mathbf{x}_1,\mathbf{x}_2,\cdots,\mathbf{x}_{B_k}\}$ and the label set of the minibatch is denoted by $\mathbf{Y}_k=\{y_1,y_2,\cdots,y_{B_k}\}$. Then each client $k\in \mathcal{C}$ performs forward propagation based on $\mathbf{w}_c$ and computes the activations $\mathbf{A}_k$ of the last layer of the client-side model as
\begin{align}\label{eqn:Ak}
    \mathbf{A}_k = h(\mathbf{w}_{c}; \tilde{\mathcal{D}}_k).
\end{align}
Then the activation $\mathbf{A}_k$ and the label set $\mathbf{Y}_k$ are uploaded to the server.
\item \textbf{Step 2 (Backpropagation and update of server-side model):} As shown in Fig. \ref{concat}, the server receives the activation sets $\{\mathbf{A}_k\}_{k \in \mathcal{C}}$ from the selected clients and concatenates them as $\mathbf{A}  \triangleq\bigcup_{k\in \mathcal{C}}\mathbf{A}_k$.
The corresponding labels for the activations are also concatenated as $\mathbf{Y} \triangleq\bigcup_{k\in \mathcal{C}}\mathbf{Y}_k$.
Then, the server performs forward propagation for the concatenated activation $\mathbf{A}$, and updates the server-side model as
\begin{align}\label{eqn:up_server}
    \mathbf{w}_s^{i+1}=\mathbf{w}_s^{i}-\eta \nabla_{\mathbf{w}_s^{i}} l(\mathbf{w}_s^{i};\mathbf{A},\mathbf{Y}),
\end{align}
where $\eta$ is the learning rate. Additionally, the server computes the backpropagated gradients as
\begin{align}\label{eqn:back}
    \mathbf{G}_k=\nabla_{\mathbf{A}_k} l(\mathbf{w}_{s}^{i};\mathbf{A}_k,\mathbf{Y}_k),
\end{align}
which is then sent back to each client $k\in \mathcal{C}$.
 \item \textbf{Step 3 (Update of client-side models):} Each client $k\in \mathcal{C}$ performs backpropagation and updates the local client-side model using chain rule \cite{rumelhart1986learning} as
\begin{align}\label{eqn:back2}
    &\mathbf{w}_{c,k}^{i+1}=\mathbf{w}_{c,k}^{i} \nonumber\\
    &~~-\eta \nabla_{\mathbf{A}_k} l(\mathbf{w}_{s}^{i};\mathbf{A}_k,\mathbf{Y}_k) \nabla_{\mathbf{w}_{c,k}^{i}} h(\mathbf{w}_{c,k}^{i};\mathbf{X}_k).
\end{align}
\end{itemize}
\par
\textbf{FL phase.} After $I$ local iterations, each client $k\in \mathcal{C}$ uploads the client-side model $\mathbf{w}_{c,k}^I$ to the server. Then the server aggregates the client-side models as 
\begin{align}\label{eqn:agg}
    \mathbf{w}_{c} = \frac{\sum_{k\in \mathcal{C}}|\mathcal{D}_k|\mathbf{w}_{c,k}^I}{\sum_{k\in \mathcal{C}}|\mathcal{D}_k|}.
\end{align}

\begin{algorithm}[t]
\caption{Training Process of SCALA}
\label{Algorithm:2}
\begin{algorithmic}[1]
\STATE {\bfseries Input:} Batchsize $B$, number of global iterations $T$, number of local iterations $I$, client set $\mathcal{K}=\{1,\cdots, K\}$, and local data $\mathcal{D}_k$ for $k\in\mathcal{K}$ 
\STATE {\bfseries Output:} Global model $\mathbf{w}^T=[\mathbf{w}_s^T, \mathbf{w}_c^T]$
\STATE {\bfseries Initialize:} Global model $\mathbf{w}^0=[\mathbf{w}_s^0, \mathbf{w}_c^0]$
\FOR{$t=1,\cdots$, $T$}
    \STATE Server randomly selects a subset of clients $\mathcal{C}^t$
    \STATE Server computes minibatch size $B_k$ of each client $k\in \mathcal{C}^t$ via (3)
    \STATE Server sends the client-side model $\mathbf{w}_c^{t}$ and minibatch size $B_k$ to each client $k\in\mathcal{C}^t$
    \FOR{$i=1,\cdots$, $I$}
        \FOR{each client $k\in \mathcal{C}^t$ \textbf{in parallel}}
            \STATE Sample a minibatch $\tilde{\mathcal{D}}_k \in \mathcal{D}_k $ of size $B_k$
            \STATE Compute activation $\mathbf{A}_k$ via (4)
            \STATE Upload activation $\mathbf{A}_k$ and label set $\mathbf{Y}_k$ to the server
        \ENDFOR
        \STATE Server concatenates $\{\mathbf{A}_k\}_{k\in\mathcal{C}^t}$ into the input and performs forward propagation
        \STATE Server adjusts the loss function via (12) and updates the server-side model $\mathbf{w}_s^{t,i}$ via (5) 
        \STATE Server adjusts the loss function via (13) and computes backpropagated gradients $\mathbf{G}_k$ of each client $k\in \mathcal{C}^t$ via (6)
        \STATE Server sends backpropagated gradients to each client $k\in \mathcal{C}^t$
        \FOR{each client $k\in \mathcal{C}^t$ \textbf{in parallel}}
            \STATE Perform backpropagation and update $\mathbf{w}_{c,k}^{t,i}$ via (7)
        \ENDFOR
    \ENDFOR
    \FOR{each client $k\in \mathcal{C}^t$ \textbf{in parallel}}
        \STATE Upload $\mathbf{w}_{c,k}^{t,I}$ to the server
    \ENDFOR
    \STATE Server aggregates $\{\mathbf{w}_{c,k}^{t,I}\}_{k\in\mathcal{C}^t}$ via (8)
\ENDFOR
\end{algorithmic}
\end{algorithm}

\subsection{Logit Adjustments in Loss Functions }
While the manner of concatenating activations offers a solution to the challenge of missing classes caused by local label distribution skew, it unveils a new obstacle: the distribution of concatenated labels often exhibits an imbalanced long-tail pattern due to the partial client participation, as shown in Fig. \ref{concat}. This distribution poses a challenge for common loss functions, as their misclassification error sensitivity varies with label frequency, where high-frequency labels result in a smaller misclassification error, leading the model to focus on improving predictive accuracy for high-frequency labels. To address this, we propose SCALA, which employs loss functions with logit adjustments for both server-side and client-side models.
\par
Specifically, consider the data distribution represented by $P(\mathbf{x},y)=P(\mathbf{x}\mid y)P(y)$. For a given data point $D = (\mathbf{x},y)$, the logit of label $y$ is denoted as $s_y(\mathbf{x})$ and the goal of standard machine learning is to minimize the misclassification error $P_{x,y}(y\neq\hat{y})$, where $\hat{y}=\arg\max_y s_y(\mathbf{x})$ is the predicted class. Since $P(y\mid \mathbf{x})\propto P(\mathbf{x}\mid y)P(y)$ according to Bayes' theorem, the class with a high $P(y)$ will achieve a reduced misclassification error. To tackle this, the balanced error is proposed by averaging each of the per-class error rates \cite{menon2021longtail}, defined as $\frac1M\sum_{y=1}^M P_{\mathbf{x}|y}(y\neq\hat{y})$, where $M$ is the total number of classes. 
In this manner, the native class-probability function $P(y\mid \mathbf{x})\propto P(\mathbf{x}\mid y)P(y)$ is calibrated to a balanced class-probability function $P^\text{bal}(y\mid \mathbf{x})\propto \frac1M P(\mathbf{x}\mid y)$, so that the varying $P(y)$ no longer affects the result of the prediction. 
\par
In this paper, we choose softmax cross-entropy to predict class probabilities and the predicted probability of label $y$ for input $\mathbf{x}$ is
\begin{align}\label{eqn:p_y}
    p_y(\mathbf{x}) = \frac{e^{s_y(\mathbf{x})}}{\sum_{y^{\prime}=1}^M e^{s_{y^{\prime}}(\mathbf{x})}}, 
\end{align}
where $p_y(\mathbf{x})\propto e^{s_y(\mathbf{x})}$ is regarded as the estimates of $P(y\mid \mathbf{x})$. Then the surrogate loss function of misclassification error can be formulated as 
\begin{align}
    g(y, s(\mathbf{x})) = -\log\left[\frac{e^{s_y(\mathbf{x})}}{\sum_{y^{\prime}=1}^M e^{s_{y^{\prime}}(\mathbf{x})}}\right].
\end{align}
When using the balanced error, the predicted class in softmax cross-entropy can be rewritten as
\begin{align}\label{eqn:pre_label}
    \arg\max_y P^\text{bal}(y\mid \mathbf{x})=\arg\max_{y}\{s_y(\mathbf{x})-\log P(y)\},
\end{align}
which indicates that the balanced class-probability function tends to reduce the logits of classes with high $P(y)$. Inspired by this, the logits for each class before softmax cross-entropy can be adjusted by \eqref{eqn:pre_label} and the softmax cross-entropy loss function with logit adjustment of the server-side model is formulated as 
\begin{align}\label{eqn:g}
    g^\text{bal}(y, s(\mathbf{x})) = -\log\left[\frac{e^{s_y(\mathbf{x})+\log P_s(y)}}{\sum_{y^{\prime}=1}^M e^{s_{y^{\prime}}(\mathbf{x})+\log P_s(y^{\prime})}}\right],
\end{align}
where $P_s(y)$ is the distribution of concatenated labels.
\par
However, the backpropagated gradients computed by \eqref{eqn:g} are not suitable for updating the client-side models due to the mismatch between the label distribution of individual clients and the concatenated label distribution, that is, $P_k(y)\neq P_s(y)$ for $k\in \mathcal{C}$. Therefore, we introduce logit adjustments for the loss functions of client-side models according to the label distribution of participating clients. 
Specifically, given a participating client $k$ along with the label distribution $P_k(y)$, the softmax cross-entropy loss function with logit adjustment of each participating client-side model is formulated as
\begin{align}\label{eqn:gk}
    g_k^\text{bal}(y, s(\mathbf{x})) = -\log\left[\frac{e^{s_y(\mathbf{x})+\log P_k(y)}}{\sum_{y^{\prime}=1}^M e^{s_{y^{\prime}}(\mathbf{x})+\log P_k(y^{\prime})}}\right].
\end{align}
\par
To summarize, we propose SCALA, which is obtained based on concatenated activations enabled SFL by introducing the loss functions with logit adjustments for server-side and client-side models. Two proposed modules are complementary and naturally aligned with SFL. Concatenated activations leverage the model split: clients execute lightweight early layers, while the server trains the deeper layers on concatenated activations, thereby addressing deep-layer bias without sharing raw data. Logit adjustments exploit the global view of the server: the server estimates the global label distribution from the concatenated activations and applies the logit adjustment to the loss function of the sever-side model, enhancing predictive performance across all labels. Compared with client-selection strategies for global balance \cite{zhang2023fed,li2024adafl,10910050}, SCALA is lightweight and scheduling-free; compared with purely local regularization or loss calibration \cite{li2020federated,lee2022preservation,zhang2022federated}, SCALA remains effective even when some clients lack entire classes, because class support is reintroduced at the activation level via concatenation.
\par
The pseudo-code of SCALA is illustrated in Algorithm \ref{Algorithm:2}. In each global iteration, the selected clients and the server collaborate to perform parallel SL over $I$ local iterations (line 8-21). In parallel SL phase, all selected clients synchronously execute local iterations (line 9-13 and line 18-20). The server-side model is centrally updated in each local iteration (line 14-15), where the loss functions for server-side and client-side models is adjust (line 15 and line 16). In FL phase, the client-side models are sent to the server for aggregation at the $I$-th local iteration (line 22-25). 

\section{Theoretical Analysis for SCALA}
In this section, we first conduct a convergence analysis of concatenated activations enabled SFL, which demonstrates how the centralized training manner of the server-side model mitigate the local label distribution skew. Subsequently, we analyze the classifier update to demonstrate how SCALA improves model performance under global label distribution skew. 

\subsection{Convergence Analysis of Concatenated Activations Enabled SFL}

Our analysis is based on the following assumptions:

\begin{assumption}\label{assumption1}

(Smoothness) The local loss functions are Lipschitz smooth, i.e., for all $\mathbf{w}$ and $\mathbf{w}'$, $\|\nabla_{\mathbf{w}} F_k(\mathbf{w}) - \nabla_{\mathbf{w}} F_k(\mathbf{w}')\| \leq \gamma \|\mathbf{w} - \mathbf{w}'\|$, where $\gamma > 0$ is the Lipschitz constant.

\end{assumption}

\begin{assumption}\label{assumption2}
(Layer-wise Bounded Gradient Variance) For each layer $n$, the stochastic gradient has a bounded variance: $\mathbb{E} \left[ \|\nabla_{\mathbf{w}_n} f_k(\mathbf{w}) - \nabla_{\mathbf{w}_n} F_k(\mathbf{w})\|^2 \right] \leq \frac{\sigma_n^2}{B_k}$.

\end{assumption}

\begin{assumption}\label{assumption3}
(Layer-wise Bounded Dissimilarity) For each layer $n$, the gradient dissimilarity is referred to as the bias caused by data heterogeneity across clients, which is bounded as: $\mathbb{E} \left[  \left\|\nabla_{\mathbf{w}_n} F(\mathbf{w})-\nabla_{\mathbf{w}_n} F_k(\mathbf{w}) \right\|^2 \right] \leq \kappa_n^2$. 
\end{assumption}

\begin{assumption}\label{assumption4}
(Bounded Inconsistency) For the server-side model, the gradient inconsistency caused by the difference between the client-side models is bounded as: $\mathbb{E} \left[ \left\|\nabla_{\mathbf{w}_s} F([\mathbf{w}_s^{t,i}; \mathbf{w}_c^{t,i}]) - \nabla_{\mathbf{w}_s} F([\mathbf{w}_s^{t,i}; \mathbf{w}_{c,k}^{t,i}])\right\|^2 \right] \leq \nu^2$. 
\end{assumption}

Bounded Gradient Variance and Bounded Dissimilarity are widely adopted in the literature for analyzing the convergence of FL algorithms \cite{wang2022unified,cho2023convergence} and we extend them to a layer-wise form to more accurately capture the impact of the split in the analysis. Assumption \ref{assumption4} characterizes the gradient inconsistency arising from the difference between client-side models during the computation of server-side model gradients. Note that when the number of client-side model layers $ N_c $ increases, the difference between the client-side models grows, which generally leads to an increase in the magnitude of this inconsistency. However, since the total gradient inconsistency is the sum of the inconsistencies across all layers of the server-side model, increasing $ N_c $ may also result in a decrease in the magnitude of the overall inconsistency. To account for these effects, we introduce a unified upper bound for gradient inconsistency, rather than a layer-wise bound.

\begin{theorem}\label{theorem1}
Under Assumptions \ref{assumption1}-\ref{assumption4}, denote $F^*=\min_{\mathbf{w}}F(\mathbf{w})$, $\sigma_{\max}^2=\max_{n\in[N]}\{\sigma_n^2\}$ and  $\kappa_{\max}^2=\max_{n\in[N]}\{\kappa_n^2\}$, let $\rho$ be the  client participation ratio, $T$ be the total global iterations and $I$ be the number of local iterations. If the learning rate $\eta = \Theta\left(\frac{1}{\sqrt{TI}}\right)$ and $\eta \leq \frac{\rho}{36\gamma I}$, then the convergence rate of Concatenated Activations Enabled SFL satisfies:
\begin{align}\label{eqn:theorem1}
&\frac{1}{T} \sum_{t=0}^{T-1} \mathbb{E}[\|\nabla_{\mathbf{w}} F(\mathbf{w}^t)\|^2] \leq \mathcal{O} \left( \frac{F(\mathbf{w}^0) - F^*}{\sqrt{T I}}+\nu^2 \right) \nonumber \\ 
&+ \mathcal{O} \left( \frac{N_s \sigma_{\max}^2}{B \sqrt{T I}} \right) + \mathcal{O} \left( \frac{N_c}{\rho\sqrt{T}}\left(\frac{ \sigma_{\max}^2}{B  \sqrt{I}} + \sqrt{I} \kappa_{\max}^2\right) \right).
\end{align}
\end{theorem}

\begin{proof}
See Appendix \ref{proof_theorem1}.
\end{proof}
The second term on the right-hand side of \eqref{eqn:theorem1} results from the stochastic gradient error introduced by the server-side model, with its magnitude decreasing as the batch size $ B $ increases. The third term arises from both the stochastic gradient error and gradient dissimilarity error introduced by the client-side model. Specifically, the gradient dissimilarity term, $ \mathcal{O} \left( \frac{N_c \sqrt{I} \kappa_{\max}^2}{\rho \sqrt{T}} \right) $, reflects the additional error due to data heterogeneity across clients, capturing the effect of local label distribution skew on convergence performance. Furthermore, the errors introduced by the client-side model are influenced by the global label distribution skew, with a lower client participation rate $ \rho $ limiting the convergence performance. 

In summary, we derive the following insights: First, concatenated activations enabled SFL achieves a sublinear convergence rate of $\mathcal{O} \left( \frac{1}{ \sqrt{T}} \right)$. Second, increasing the number of server-side model layers $N_s$  can mitigate the impact of local label distribution skew by reducing the gradient dissimilarity term. Third, increasing the client participation rate $\rho$ can alleviate the effect of global label distribution skew on convergence performance.
\subsection{Analysis on Update Process of the Classifier of SCALA}

We denote $\pi(\mathbf{x})$ as the model feature for a given input $(\mathbf{x},y)$ and $\zeta=[\zeta_1,\zeta_2, \cdots, \zeta_M]$ as the weight matrix of the classifier, where $\zeta_y$ represents the classifier of label $y$. Then the logit of label $y$ for a given input $(\mathbf{x},y)$ is calculated as $s_y(\mathbf{x}) = \zeta_y\cdot\pi(\mathbf{x})$, where $\cdot$ represents the dot product operator.
Denote $\Delta \zeta_y = \zeta_y^{\text{new}}-\zeta_y^{\text{old}}$ as the update of the classifier. According to \eqref{eqn:p_y}, the update of the classifier should increase the logit, that is, make $\Delta \zeta_y\cdot\pi(\mathbf{x})>0$. Then $p_y(\mathbf{x})$ will increase to improve the prediction accuracy. We concentrate on the impact of label distribution on the update process of a classifier. To this end, we consider an ideal model feature extraction layer which ensures that the features of different labels are orthogonal to each other, as demonstrated below:
\begin{assumption}\label{assumption5}
Given a dataset $\mathcal{D}$ with $M$ classes, the model features $\pi(\mathbf{x})$ satisfy $\pi_y \cdot \pi_{y^{\prime}} = 0$ for all $y \neq y^{\prime}$, where $\pi_y$ is the averaged model features of label $y$ defined as $\pi_y=\frac{1}{|\mathcal{D}_y|}\sum_{\mathbf{x}_i\in\mathcal{D}_y}\pi(\mathbf{x}_i)$ and $\mathcal{D}_y$ is the subset of dataset $\mathcal{D}$ with label $y$. 
\end{assumption}

Then we propose the following theorem:

\begin{theorem}\label{theorem2}
Under Assumptions \ref{assumption5}, when $P(y)$ approaches $0$, the update of the logit of label $y$ satisfies
\begin{align}
\Delta\zeta^{\text{bal}}_y\cdot\pi_y > \Delta\zeta_y\cdot\pi_y,
\end{align}
when $P(y)$ approaches $1$, the update of the logit of label $y$ satisfies
\begin{align}
\Delta\zeta^{\text{bal}}_y\cdot\pi_y < \Delta\zeta_y\cdot\pi_y 
\end{align}
\end{theorem}
\begin{proof}
See Appendix \ref{proof_theorem2}.
\end{proof}
Theorem \ref{theorem2} provides us with an insight: a loss function with logit adjustment actually works by sacrificing the recognition for high-frequency labels to enhance the recognition of low-frequency labels.

\section{Experiments}

\subsection{Experiment Setting}
\textbf{Implementation details.} 
Unless otherwise stated, we set up $100$ clients and the server randomly selects $\rho=10\%$ of the total clients at each global iteration. We use AlexNet \cite{krizhevsky2012imagenet} and ResNet-18 \cite{he2016deep} as the model. For AlexNet, the split point is selected at the second convolutional layer, while for ResNet-18, it is selected at the first residual block. The size of minibatch $B$ for the server-side model is $320$ and the number of local SGD iterations is $20$. The model is updated via SGD optimizer with learning rate $\eta=0.01$. We run each experiment with $3$ random seeds and report the average accuracy.
To simulate a real communication environment, we consider a cellular network with a radius of $1000$ meters. The server is positioned at the center, while the clients are distributed randomly and uniformly within the network. The path loss between each client and the server is calculated as $128.1 + 37.6 \log_{10}(r)$ dB, where $r$ is the distance from the client to the server in kilometers, according to \cite{abetaevolved}. Each client transmits at a uniform power of $0.2$ W. The uplink channels are assumed to be orthogonal, with a total bandwidth of $10$ MHz and an additive Gaussian noise power spectral density of $-174$ dBm/Hz. Additionally, the clients are assigned random computational capacities ranging from $10^9$ to $5\times10^{9}$ FLOPs. 
\par
\textbf{Baselines.}
We choose several typical methods designed to address data heterogeneity, including FedProx \cite{li2020federated}, FedLC \cite{zhang2022federated}, FedCBS \cite{zhang2023fed} and FedConcat \cite{diao2024exploiting}, as baselines. Furthermore, we compare SCALA with several SFL algorithms, including traditional SFL methods like SplitFedV1, SplitFedV2 \cite{thapa2022splitfed}, and SFLLocalLoss \cite{han2021accelerating}, as well as SFL algorithms designed to address data heterogeneity, such as Minibatch-SFL \cite{huang2023minibatch} and CS-SFL \cite{10910050}.

\par
\textbf{Dataset.} 
We adopt three popular image classification benchmark datasets, namely CIFAR10 \cite{krizhevsky2009learning}, CIFAR100 \cite{krizhevsky2009learning}, CINIC10 \cite{darlow2018cinic} and ImageNette \cite{deng2009imagenet}. 
To simulate the label skew distribution and generate local data for each client, we consider two label skew settings \cite{zhang2022federated}: 1) \emph{quantity-based label skew} and 2) \emph{distribution-based label skew}. In quantity-based label skew, given $K$ clients and $M$ classes, we divide the data of each label into $\frac{K\cdot \alpha}{M}$ portions, and then randomly allocate $\alpha$ portions of data to each client. Consequently, each client can obtain data from at most $\alpha$ classes, indicating that there is the presence of class missing in the client data. We use $\alpha$ to represent the degree of label skew, where a smaller $\alpha$ implies stronger label skewness. In distribution-based label skew, for each selected client $k$, we use Dirichlet distribution $\text{Dir}_M(\beta)$ with $M$ classes to sample a probability vector $\mathbf{p}_k(p_{k,1}, p_{k,2},\cdots, p_{k,M})\sim \text{Dir}_M(\beta)$ and allocate a portion of $p_{k,y}$ of the samples in class $y$ to client $k$. We use $\beta$ to denote the degree of skewness, where a smaller $\beta$ implies stronger label skewness.

\begin{figure}[t]
    \centering 
    \subfigure[CIFAR10 with $\alpha=2$.]{ 
    \includegraphics[clip, viewport= 5 0 398 298,width=0.227\textwidth]{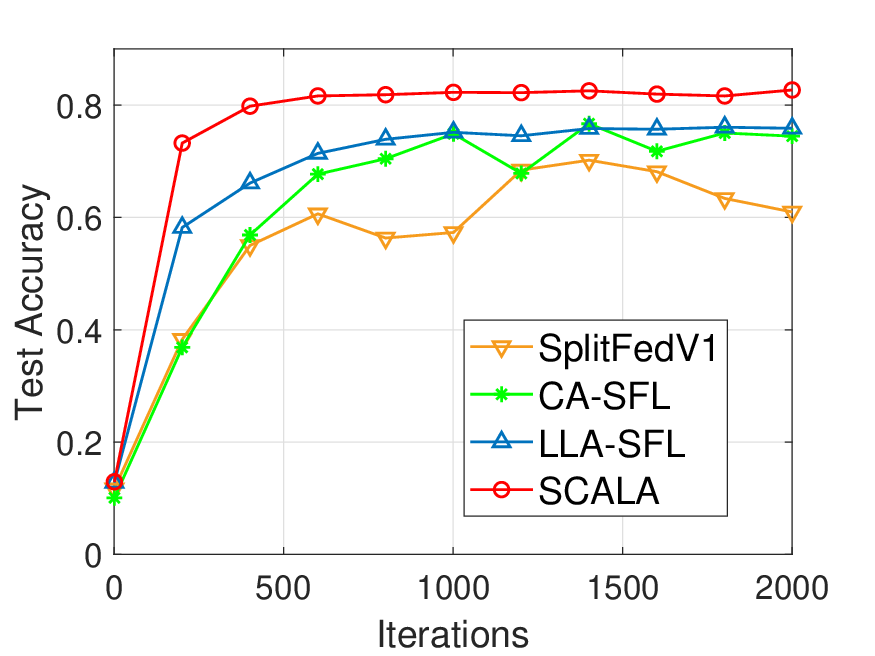} 
    } 
    \subfigure[CIFAR10 with $\beta=0.05$.]{
    \includegraphics[clip, viewport= 5 0 398 298,width=0.227\textwidth]{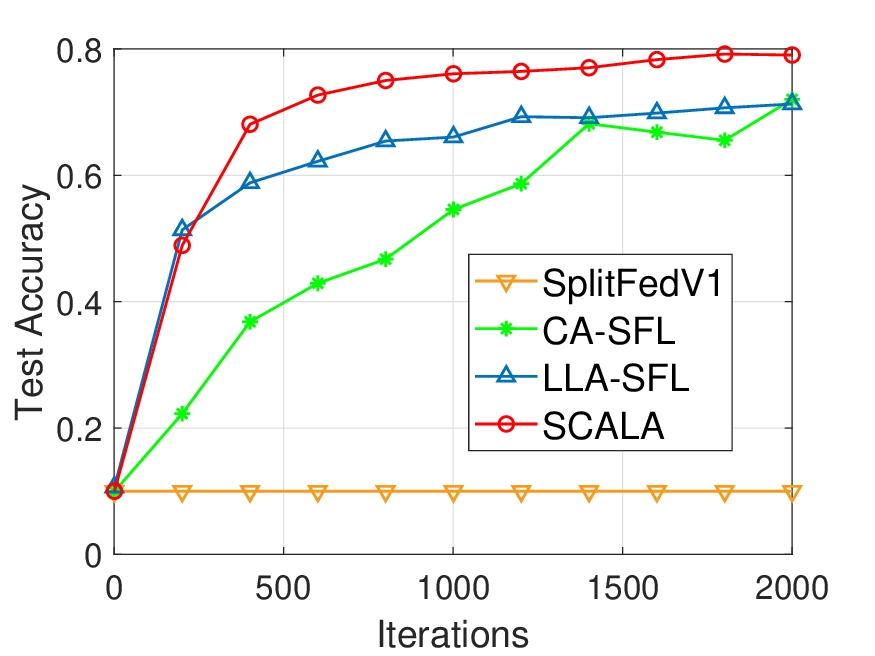} 
    }
\caption{Test accuracy of SCALA compared with alternative SFL configurations.} 
\label{Abla}
\end{figure}

\begin{table}[t]
\caption{Test accuracy (\%) on CIFAR10 under different client participation ratios with alternative SFL configurations.}
\label{ablation_1}
\resizebox{\linewidth}{!}{
\begin{tabular}{cccccc}
\toprule
Skewness                        & Method     & $\rho=5\%$         & $\rho=10\%$        & $\rho=20\%$        & $\rho=50\%$        \\ \midrule
\multirow{4}{*}{$\alpha=2$}     & SplitFedV1 & -                  & $60.97_{\pm 7.43}$ & $71.34_{\pm 2.65}$ & $76.63_{\pm 2.45}$ \\
                                & CA-SFL     & -                  & $74.47_{\pm 2.59}$ & $77.01_{\pm 1.95}$ & $82.18_{\pm 0.50}$ \\
                                & LLA-SFL    & $74.87_{\pm 0.30}$ & $75.87_{\pm 0.96}$ & $77.79_{\pm 0.15}$ & $78.42_{\pm 0.65}$ \\
                                & SCALA       & $78.66_{\pm 0.58}$ & $82.70_{\pm 0.57}$ & $83.69_{\pm 0.59}$ & $84.16_{\pm 0.43}$ \\ \midrule
\multirow{4}{*}{$\beta = 0.05$} & SplitFedV1 & -                  & -                  & $68.91_{\pm 4.70}$ & $76.62_{\pm 1.95}$ \\
                                & CA-SFL     & -                  & $72.42_{\pm 1.17}$ & $74.23_{\pm 1.61}$ & $79.21_{\pm 0.71}$ \\
                                & LLA-SFL    & $69.64_{\pm 1.01}$ & $71.29_{\pm 2.94}$ & $72.02_{\pm 1.48}$ & $73.74_{\pm 1.43}$ \\
                                & SCALA       & $74.67_{\pm 3.36}$ & $79.04_{\pm 1.33}$ & $81.02_{\pm 0.43}$ & $81.99_{\pm 0.41}$ \\ \bottomrule
\end{tabular}}
\end{table}

\begin{figure}[t]
    \centering 
    \subfigure[CIFAR10 with $\beta=0.1$.]{ 
    \includegraphics[width=0.47\textwidth]{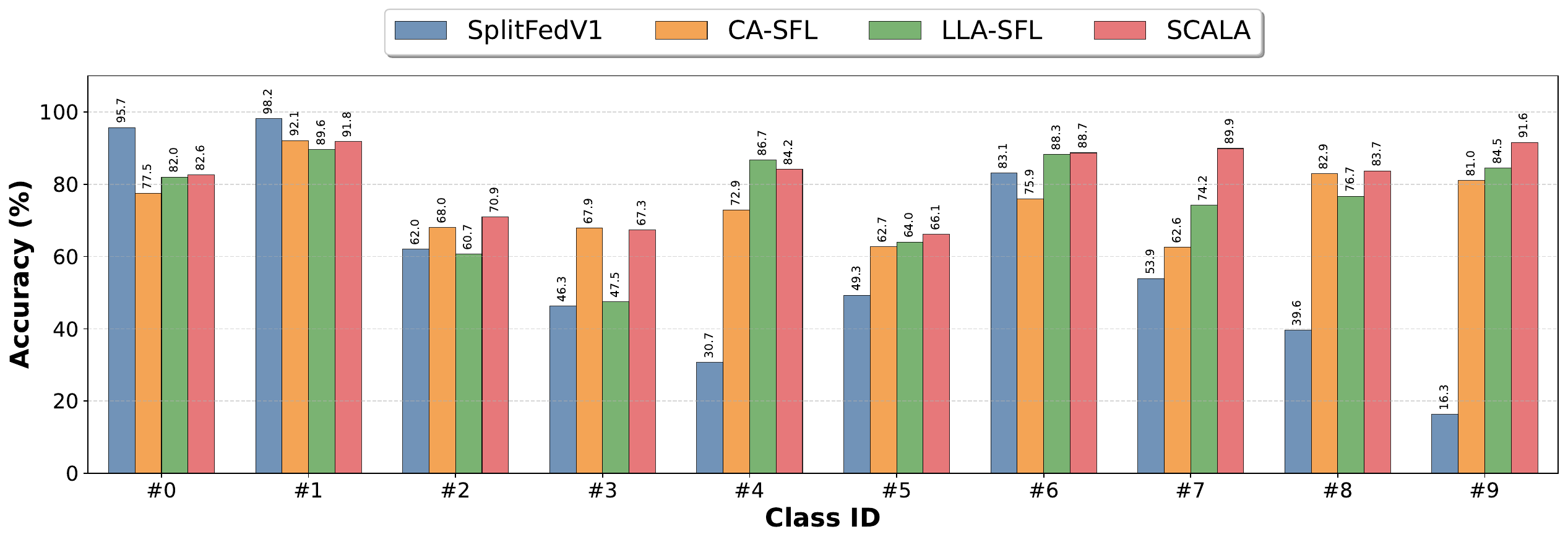} 
    } 
    \subfigure[CIFAR10 with $\beta=0.3$.]{
    \includegraphics[width=0.47\textwidth]{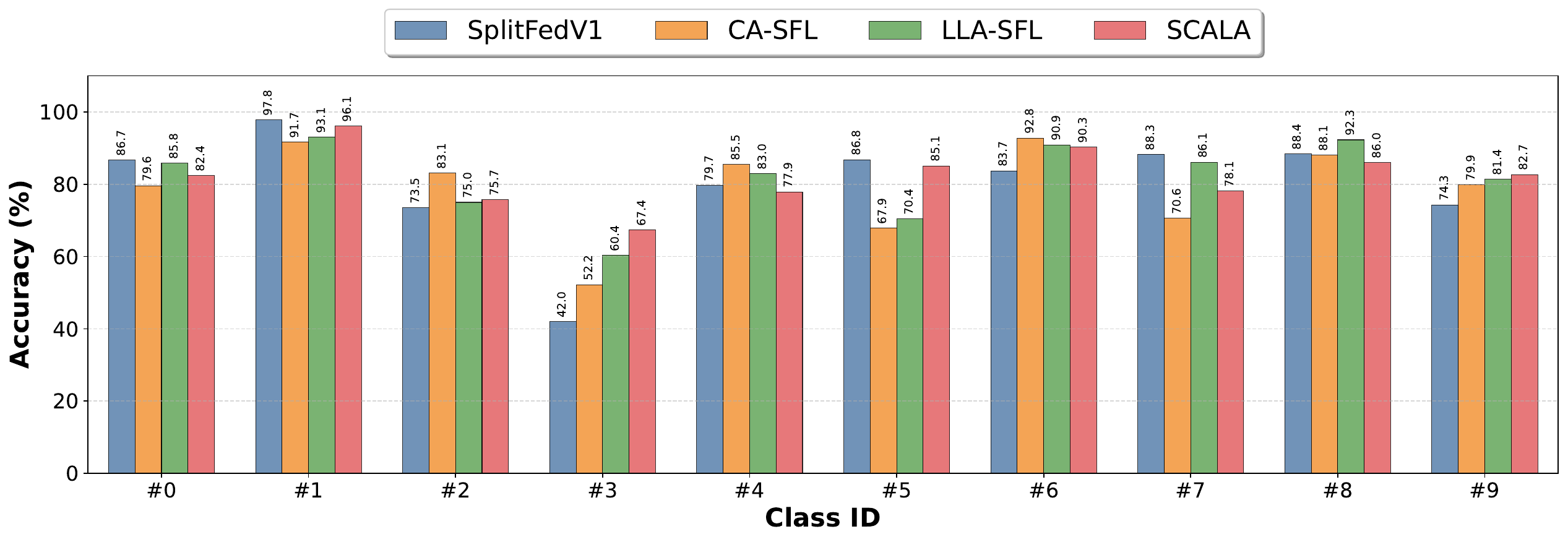} 
    }
\caption{Per-class test accuracy of SCALA compared with alternative SFL configurations.} 
\label{Abla_per_class}
\end{figure}

\subsection{Ablation Study of SCALA}

We begin by conducting an ablation study of SCALA using AlexNet to validate the individual contributions of our core mechanisms. We evaluate three distinct configurations alongside the proposed framework: (1) the traditional SFL algorithm, SplitFedV1; (2) SFL enabled only with Concatenated Activations, denoted as CA-SFL; and (3) SFL with only Local Logit Adjustments, denoted as LLA-SFL. The experimental results are presented in Fig. \ref{Abla}. As shown in the figure, both CA-SFL and LLA-SFL demonstrate significant improvements in convergence performance compared to SplitFedV1. This indicates that the centralized update manner based on concatenated activations and the application of logit adjustments to the local loss functions can enhance model performance. Furthermore, SCALA achieves the best convergence performance, highlighting that incorporating logit adjustments into the global loss function—based on concatenated activations—effectively mitigates the challenges posed by global label distribution skew and further optimizes the model performance.
\par
To provide a deeper insight into the efficacy of concatenated activations, we conduct a fine-grained analysis by varying the client participation ratio $\rho$ to assess the sensitivity of the overall accuracy to the scale of concatenation. The results, summarized in Table \ref{ablation_1}, reveal a positive correlation between $\rho$ and model performance, where `-' indicates the model fails to converge.  It can be observed that increasing the participation ratio enables the concatenated activations to encompass category information from a broader range of sources. This advantage is reflected in the performance of CA-SFL, which exhibits a substantial accuracy gain as the ratio rises—a trend that is much more pronounced compared to LLA-SFL.  This enriched representation effectively alleviates local label distribution skew, thereby leading to superior overall accuracy, particularly at higher participation ratios.
\par
To provide a deeper insight into the efficacy of logit adjustments, we conduct a fine-grained analysis by evaluating the per-class recognition accuracy at global iteration $1000$ under varying degrees of label skewness, specifically $\beta=0.1$ and $\beta=0.3$. The results, as illustrated in Fig. \ref{Abla_per_class}, demonstrate that the recognition capability across different classes becomes increasingly biased as label skewness intensifies. For instance, in the $\beta=0.1$ scenario, the baseline SplitFedV1 exhibits poor accuracy ($<40\%$) on Classes 4, 8, and 9. Although LLA-SFL attempts to mitigate this by applying local logit adjustments, it yields only marginal performance gains. In contrast, SCALA achieves the most balanced recognition capability among all baselines. This result reveals a vital insight regarding the logit adjustments module: its efficacy is fundamentally predicated on the concatenated activations. This dependence arises because individual clients in LLA-SFL lack access to the global label distribution, rendering local adjustments less effective. By contrast, SCALA constructs a ``virtual centralized" training batch via concatenation, providing the server with a holistic view of the global long-tailed distribution. It is only upon this concatenated activations that logit adjustments can accurately calibrate the loss function.


\subsection{Effect of Data Heterogeneity}
\label{BaselineExp}
We evaluate SCALA against baseline methods under varying degrees of label skewness, employing AlexNet as the backbone model. For the CIFAR10, CIFAR100, and CINIC10 datasets, the training process is conducted for $2000$ global iterations. As shown in Table \ref{table1}, SCALA significantly improves the model accuracy across various degrees of label skewness, particularly in settings where some classes of local data are missing, that is, under $\alpha =2$ and $\beta=0.05$ configurations. 
Additionally, we compare the performance of SCALA with the baseline methods in a real communication environment. As shown in Fig. \ref{figure2}, SCALA not only improves model accuracy but also accelerates the convergence speed.
The improvements demonstrated in the experimental results can be attributed to the use of concatenated activations and logit adjustments in SCALA. On one hand, SCALA conducts centralized logit adjustments on the concatenated activations, resulting in higher accuracy compared to FedLC, which only perform logit adjustments on distributed clients. This demonstrates the benefits of concatenating activations. On the other hand, by applying logit adjustments to the skewed concatenated labels, SCALA achieves higher accuracy than Minibatch-SFL, which only relies on concatenated activations. This highlights the additional gains from logit adjustments.

\begin{table*}[t]
\caption{Test accuracy (\%) on CIFAR10, CINIC10 and CIFAR100.}
\label{table1}
\centering
\resizebox{\linewidth}{!}{
\begin{tabular}{ccccccccc}
\toprule
\multirow{2}{*}{Method} & \multicolumn{2}{c}{CIFAR10}                               & \multicolumn{4}{c}{CINIC10}                                                                                           & \multicolumn{2}{c}{CIFAR100}                                                                 \\
\cmidrule(lr){2-3}\cmidrule(lr){4-7}\cmidrule(lr){8-9}
                        & $\alpha=2$                  & $\beta=0.05$                & $\alpha=2$                  & $\alpha=4$                  & $\beta=0.05$                & $\beta=0.1$                 & $\alpha=2$                                                     & $\beta=0.05$                \\ \midrule
FedAvg                  & $68.36_{\pm 2.93}$          & $35.59_{\pm 1.66}$          & $52.76_{\pm 5.48}$          & $58.40_{\pm 3.73}$          & $21.89_{\pm 6.74}$          & $46.36_{\pm 7.57}$          & $24.68_{\pm 2.24}$                                             & $46.86_{\pm 0.98}$          \\
FedProx                 & $69.03_{\pm 0.94}$          & $60.00_{\pm 6.39}$          & $53.63_{\pm 5.96}$          & $60.90_{\pm 2.45}$          & $26.11_{\pm 8.22}$          & $47.04_{\pm 8.86}$          & $24.95_{\pm 1.21}$                                             & $47.33_{\pm 1.46}$          \\
FedLC                   & $76.91_{\pm 0.36}$          & $68.46_{\pm 3.17}$          & $61.45_{\pm 2.66}$          & $69.25_{\pm 0.44}$          & $47.62_{\pm 1.18}$          & $57.05_{\pm 0.57}$          & $23.35_{\pm 0.32}$                                             & $48.75_{\pm 0.46}$          \\ 
FedCBS                   & $68.54_{\pm 4.67}$          & $54.83_{\pm 1.94}$          & $46.20_{\pm 2.89}$          & $65.75_{\pm 0.94}$          & $41.27_{\pm 4.22}$          & $55.74_{\pm 3.62}$          & $21.73_{\pm 4.13}$                                             & $47.04_{\pm 0.23}$          \\ 
\midrule
SplitFedV1              & $60.97_{\pm 7.43}$          & -                           & $42.29_{\pm 12.80}$         & $58.24_{\pm 2.16}$          & -                           & $39.63_{\pm 8.10}$          & $20.67_{\pm 3.01}$                                             & $47.28_{\pm 0.35}$          \\
SplitFedV2              & $62.24_{\pm 5.51}$          & -                           & -                           & $55.61_{\pm 2.85}$          & -                           & $47.77_{\pm 3.07}$          & -                                                              & $28.01_{\pm 2.29}$          \\
SFLLocalLoss            & $71.56_{\pm 2.92}$          & $73.84_{\pm 1.45}$          & $56.13_{\pm 6.93}$          & $67.68_{\pm 0.58}$          & $53.92_{\pm 3.26}$          & $60.36_{\pm 1.84}$          & $16.56_{\pm 0.63}$                                             & $49.68_{\pm 1.26}$          \\
Minibatch-SFL           & $74.47_{\pm 2.59}$          & $72.42_{\pm 1.17}$          & $55.46_{\pm 4.34}$          & $66.89_{\pm 1.58}$          & $29.65_{\pm 3.45}$          & $62.55_{\pm 2.20}$          & -                                                              & $47.79_{\pm 1.11}$          \\
CS-SFL                  & $77.13_{\pm 1.67}$          &$72.95_{\pm 2.28}$          & $61.45_{\pm 2.57}$          &$64.50_{\pm 2.05}$          & $53.90_{\pm 2.33}$          & $60.62_{\pm 0.70}$          & $22.22_{\pm 0.24}$                                                              & $41.63_{\pm 0.37}$          \\
\midrule
SCALA                    & $\mathbf{82.70}_{\pm 0.57}$ & $\mathbf{79.04}_{\pm 1.33}$ & $\mathbf{68.96}_{\pm 1.41}$ & $\mathbf{71.89}_{\pm 0.27}$ & $\mathbf{55.67}_{\pm 5.43}$ & $\mathbf{65.34}_{\pm 1.09}$ & $\mathbf{45.46}_{\pm 1.20}$ & $\mathbf{54.73}_{\pm 0.37}$ \\ \bottomrule
\end{tabular}}
\end{table*}

\begin{figure}[t]
  \centering 
  \subfigure[CIFAR10 with $\alpha=2$.]{ 
    \includegraphics[clip, viewport= 5 0 398 298,width=0.227\textwidth]{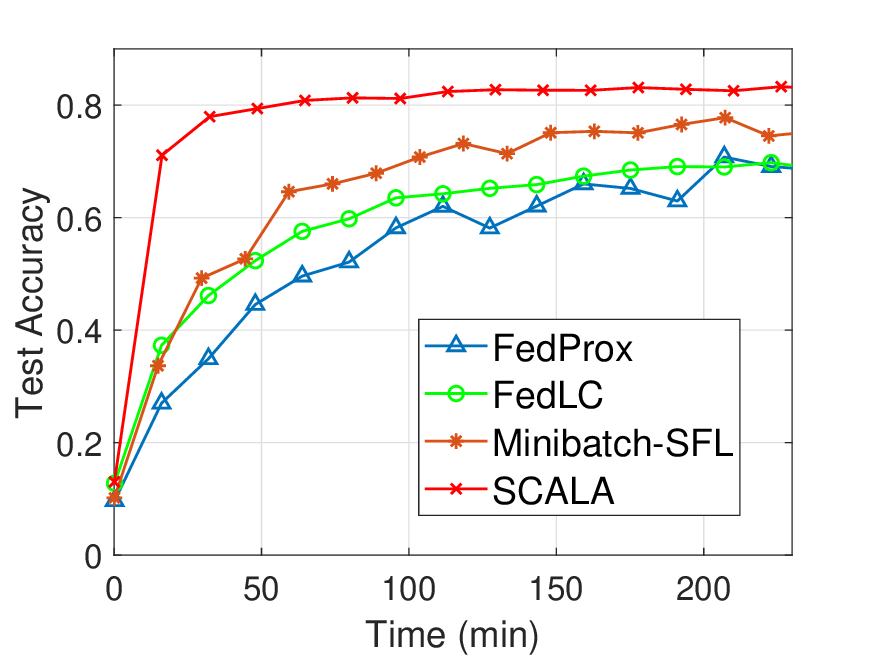} 
  } 
  \subfigure[CIFAR10 with $\beta=0.05$.]{ 
    \includegraphics[clip, viewport= 5 0 398 298,width=0.227\textwidth]{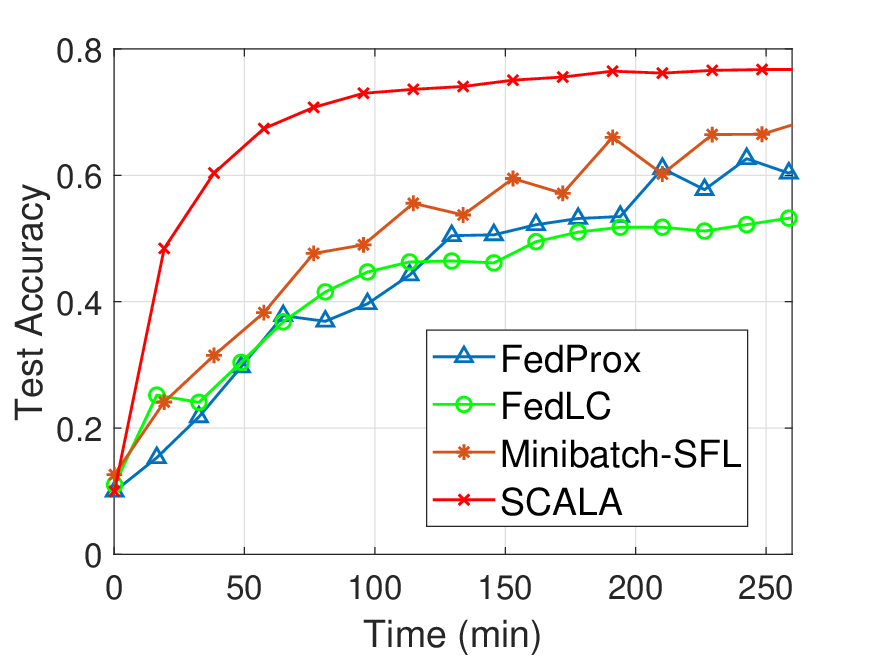} 
  } 
    \subfigure[CIFAR100 with $\alpha=2$.]{ 
    \includegraphics[clip, viewport= 5 0 398 298,width=0.227\textwidth]{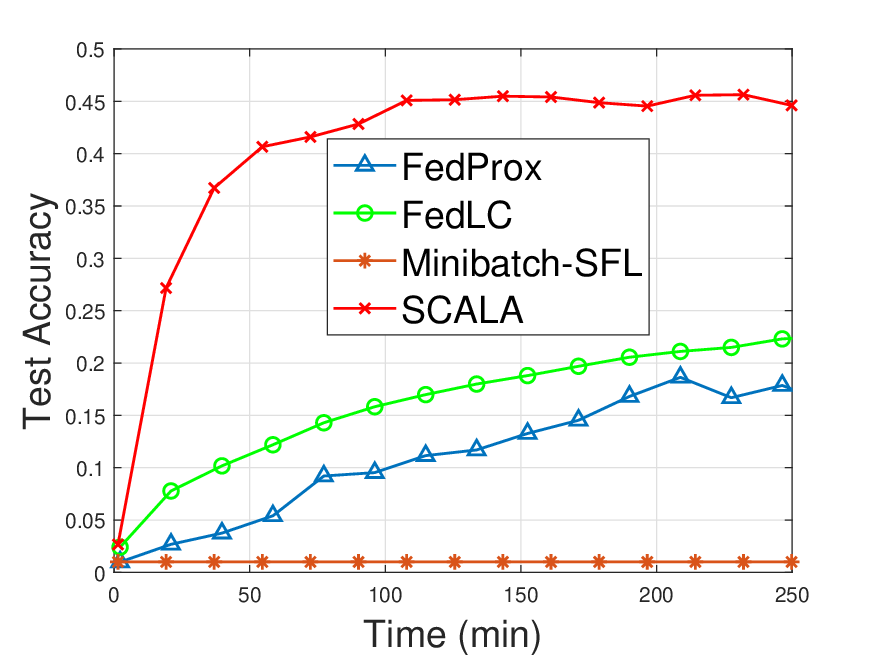} 
  } 
  \subfigure[CIFAR100 with $\beta=0.05$.]{ 
    \includegraphics[clip, viewport= 5 0 398 298,width=0.227\textwidth]{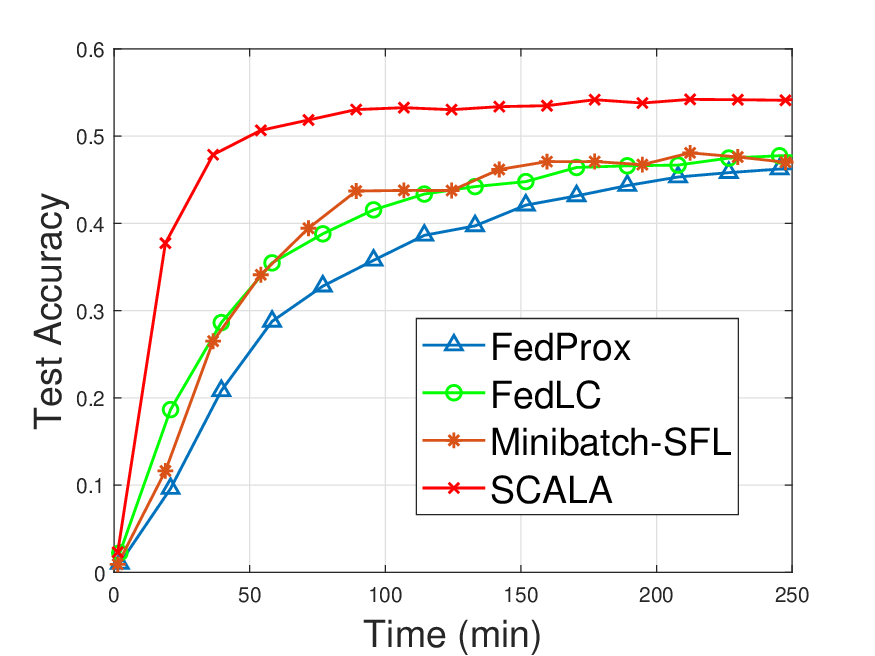} 
  } 
  \subfigure[ImageNette with $\alpha=2$.]{ 
    \includegraphics[clip, viewport= 5 0 398 298,width=0.227\textwidth]{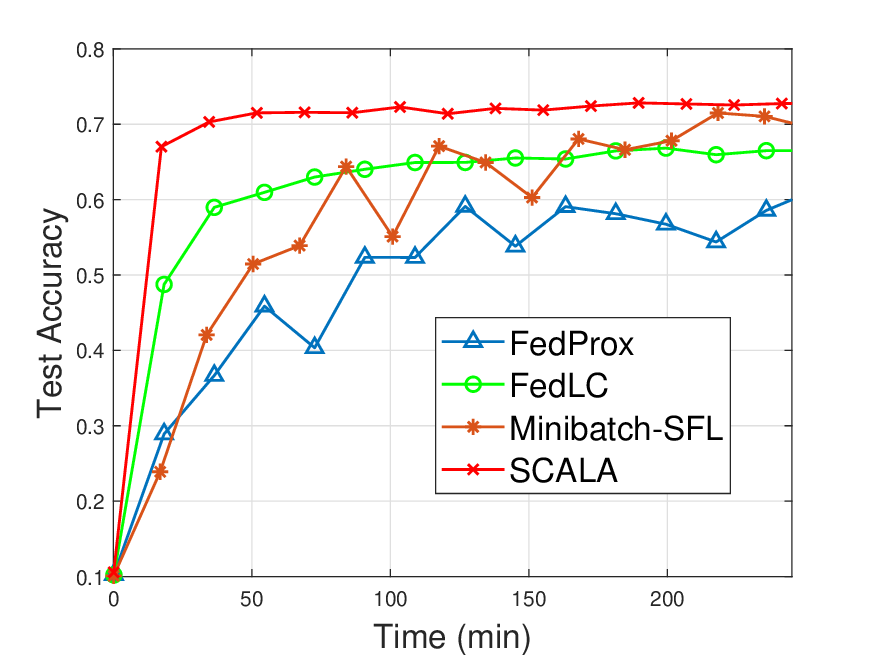} 
  } 
  \subfigure[ImageNette with $\beta=0.1$.]{ 
    \includegraphics[clip, viewport= 5 0 398 298,width=0.227\textwidth]{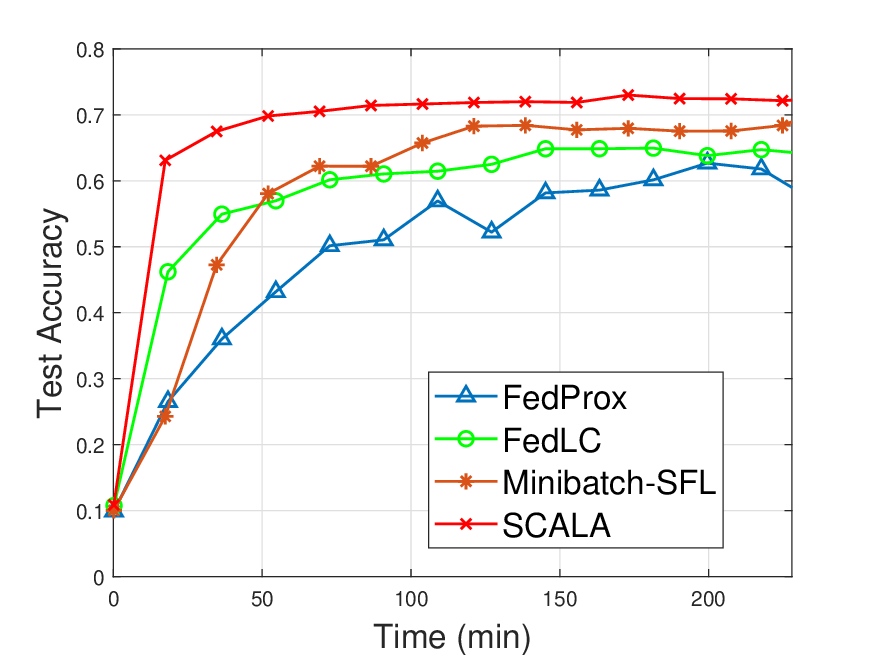} 
  } 
\caption{Test accuracy of SCALA compared against the baseline methods.} 
\label{figure2}
\end{figure}

\begin{table}[t]
\caption{Test accuracy (\%) on CIFAR-100 under different client participation ratios compared against baselines.}
\label{table4}
\resizebox{\linewidth}{!}{
\begin{tabular}{cccccc}
\toprule
Skewness                        & Method    & $\rho=5\%$                     & $\rho=10\%$                    & $\rho=20\%$                    \\ \midrule
\multirow{4}{*}{$\alpha=3$}     & FedProx   & $8.67_{\pm 0.47}$           & $11.03_{\pm 1.23}$          & $20.90_{\pm 1.08}$          \\
                                & FedLC     & $35.92_{\pm 0.16}$          & $37.94_{\pm 0.50}$          & $39.28_{\pm 1.16}$          \\
                                & FedConcat & $6.93_{\pm 0.14}$           & $17.74_{\pm 2.55}$          & $35.48_{\pm 1.01}$          \\ \cmidrule{2-5}   
                                & SCALA      & $\mathbf{58.17}_{\pm 1.15}$ & $\mathbf{62.26}_{\pm 0.62}$ & $\mathbf{63.50}_{\pm 0.62}$ \\ \midrule
\multirow{4}{*}{$\beta = 0.05$} & FedProx   & $54.78_{\pm 2.71}$          & $61.70_{\pm 1.47}$          & $66.22_{\pm 0.31}$          \\
                                & FedLC     & $54.78_{\pm 2.71}$          & $63.06_{\pm 0.48}$          & $65.80_{\pm 0.39}$          \\
                                & FedConcat & $29.22_{\pm 0.67}$          & $44.88_{\pm 0.86}$          & $58.22_{\pm 0.97}$          \\ \cmidrule{2-5}   
                                & SCALA      & $\mathbf{65.26}_{\pm 0.32}$ & $\mathbf{66.22}_{\pm 0.38}$ & $\mathbf{66.05}_{\pm 0.87}$ \\ \bottomrule
\end{tabular}}
\end{table}

\subsection{Effect of Partial Client Participation}
We study the robustness of SCALA to the proportion of clients participating at each global iteration using ResNet-18. We select participation ratios of $\rho = 5\%$, $\rho = 10\%$, and $\rho = 20\%$ and set the degree of label skewness as $\alpha =3$ and $\beta =0.05$. The results are shown in Table \ref{table4}. 
We observe that SCALA exhibits high robustness to the variation in client participation ratios, maintaining high accuracy across all settings.
Note that under the distribution-based label skew setting, higher client participation ratios, such as $\rho = 20\%$, lead to improved accuracy across all methods. This is because the increased client participation ensures sufficient training of data from each class in each iteration, thereby mitigating the impact of global label distribution skew. On the other hand, under the quantity-based label skew setting, the accuracy of all baseline methods is notably low, especially for FedConcat. This indicates that the presence of missing classes in local datasets significantly degrades model performance.
SCALA addresses this issue by introducing concatenated activations and further alleviating global label distribution skew through logit adjustments, thereby achieving higher accuracy even in spectrum-constrained environments with lower participation rates. 

\begin{table}[t]
\caption{Comparison of communication and computation overhead per global iteration, as well as test accuracy under varying local iterations on CIFAR10.}
\label{comm}
\centering
\resizebox{\linewidth}{!}{
\begin{tabular}{ccccc}
\toprule
$I$                   & Method  & Accuracy (\%)              & Comm. Overhead               & Comp. Overhead         \\ \midrule
\multirow{3}{*}{$10$} & FedProx & $59.02_{\pm7.04}$          & $14.76$ MB                   & $44.69$ GFLOPs                 \\ \
                      & FedLC   & $70.70_{\pm1.56}$          & $14.76$ MB                   & $44.69$ GFLOPs              \\ \cmidrule{2-5} 
                      & SCALA   & $\mathbf{80.33}_{\pm0.34}$ & $\mathbf{5.07}$ MB           & $\mathbf{3.59}$ GFLOPs         \\ \midrule
\multirow{3}{*}{$20$} & FedProx & $59.12_{\pm7.83}$          & $14.76$ MB                   & $89.38$ GFLOPs                 \\ 
                      & FedLC   & $73.84_{\pm1.55}$          & $14.76$ MB                   & $89.38$ GFLOPs               \\ \cmidrule{2-5} 
                      & SCALA   & $\mathbf{81.89}_{\pm0.23}$ & $\mathbf{10.07}$ MB          & $\mathbf{7.18}$ GFLOPs         \\ \midrule
\multirow{3}{*}{$30$} & FedProx & $64.67_{\pm 7.24}$         & $14.76$ MB                   & $134.07$ GFLOPs                 \\
                      & FedLC   & $75.20_{\pm 0.49}$         & $14.76$ MB                   & $134.07$ GFLOPs             \\  \cmidrule{2-5}
                      & SCALA    & $\mathbf{81.39}_{\pm0.54}$ & $\mathbf{15.08}$ MB          & $\mathbf{10.77}$ GFLOPs         \\ \bottomrule
\end{tabular}}
\end{table}

\subsection{Communication Overhead Analysis of SCALA}
We conduct experiments to analyze the communication overhead introduced by SCALA. Specifically, The total communication overhead of SCALA in each global iteration can be expressed as $ D(A) \cdot I + D(\mathbf{w}_c) $, where $ D(\cdot) $ denotes the size of the transmitted paremeters and $ I $ is the number of local iterations. In widely used models such as AlexNet and VGG, the size of activations is significantly smaller than that of the full model parameters $ D(\mathbf{w}) $. Therefore, as long as $ D(A) \cdot I + D(\mathbf{w}_c) < D(\mathbf{w}) $, SCALA incurs lower communication overhead compared to traditional FL approaches. 
To illustrate this advantage more concretely, we compare the communication overhead, computation overhead, and model accuracy of SCALA with baseline methods under different numbers of local iterations. The experiments are conducted using AlexNet on the CIFAR10, with a skewness parameter set to $\alpha=2$ and the number of global iterations fixed at $1000$. As shown in the table \ref{comm}, SCALA achieves higher accuracy with fewer local iterations, while also reducing both communication and computation overhead compared to baselines that transmit and train full model weights. These results demonstrate that SCALA enables a flexible trade-off between communication efficiency and model performance by adjusting the number of local iterations.

\subsection{Effect of Split Point Selection}
\begin{figure}[t]
  \centering 
  \subfigure[AlexNet.]{ 
    \includegraphics[clip, viewport= 5 0 398 298,width=0.227\textwidth]{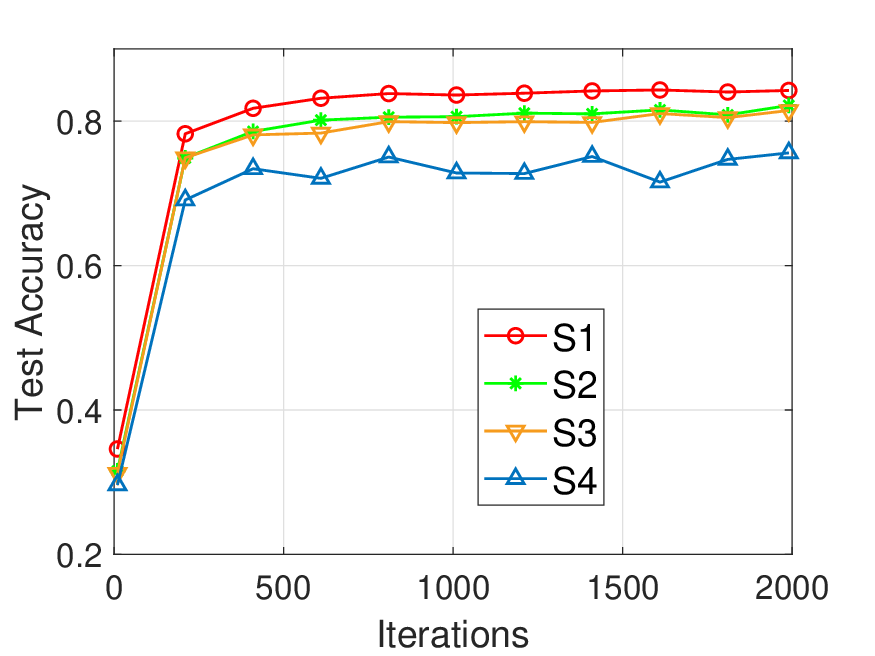} 
  } 
  \subfigure[ResNet-18.]{ 
    \includegraphics[clip, viewport= 5 0 398 298,width=0.227\textwidth]{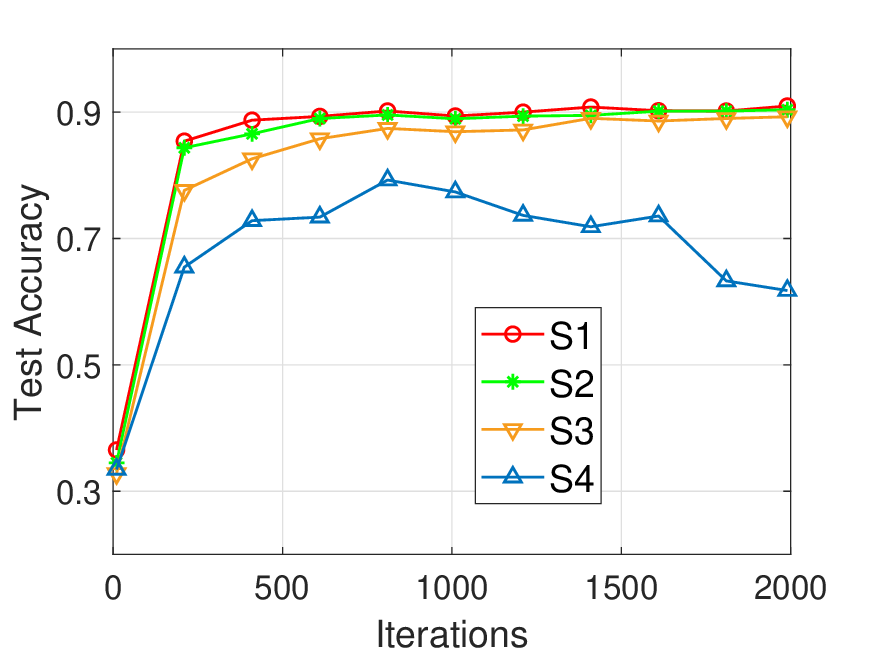} 
  } 
  \caption{Test accuracy of SCALA under different split points on CIFAR10 with label skewness $\alpha=2$.}
\label{split}
\end{figure}

We conducted ablation studies on different split points in AlexNet and ResNet-18. For AlexNet, the split points are selected at the first convolutional layer, the third convolutional layer, the fourth convolutional layer, and the fifth convolutional layer. For ResNet-18, the split points are selected at the first convolutional layer, the first residual block, the second residual block, and the third residual block. We denote the split points from shallow to deep as S1, S2, S3, and S4. The experimental results are shown in Fig. \ref{split}.
As indicated in Fig. \ref{split}, the model accuracy decreases as the depth of the split increases. This result demonstrates that deploying more models on the server to use concatenated activations for centralized training can effectively mitigate model drift in local label distribution skew, thereby improving model performance.

\begin{table}[t]
\caption{Performance of privacy-enhanced SCALA (SCALA-PE) in terms of image reconstruction quality and model accuracy.}
\label{privacy_table}
\centering
\resizebox{\linewidth}{!}{
\begin{tabular}{cccccc}
\toprule
Dataset                   & Method   & Accuracy (\%) & MSE    & SSIM   & PSNR  \\ \midrule
\multirow{2}{*}{CIFAR10}  & SCALA    & 91.25         & 0.0016 & 0.9314 & 28.20 \\
                          & SCALA-PE & 89.08         & 0.0051 & 0.8142 & 23.33 \\ \midrule
\multirow{2}{*}{CIFAR100} & SCALA    & 67.35         & 0.0048 & 0.8632 & 23.81 \\
                          & SCALA-PE & 65.50         & 0.0328 & 0.5659 & 15.84 \\ \bottomrule
\end{tabular}}
\end{table}

\subsection{Enhancing Privacy in SCALA}
We conduct experiments to evaluate the effectiveness of the proposed privacy-enhanced SCALA framework. Similar to other SL and SFL methods, SCALA requires transmitting intermediate activations from the clients to the server, which makes it vulnerable to model inversion attacks. The attacker may exploit these activations to reconstruct the original input data, thereby causing potential privacy leakage. To address this threat, we incorporate two complementary privacy-preserving mechanisms from ResSFL \cite{li2022ressfl} into SCALA. First, a simulated inversion model is employed to mimic model inversion attacks and generate reconstruction scores. These scores are then utilized as a regularization term to guide the client-side feature extractor toward producing privacy-preserving representations. The overall optimization objective is formulated as
\begin{align}
f_k(\mathbf{w}_c)
=l(\mathbf{w}_s;h(\mathbf{w}_c;\tilde{\mathcal{D}}_k))+\gamma \mathcal{R}(\mathcal{L}(\mathbf{w}_v;h(\mathbf{w}_c;\tilde{\mathcal{D}}_k)),\tilde{\mathcal{D}}_k),
\end{align}
where $\mathbf{w}_v$ denotes the parameters of the simulated inversion model, and $\mathcal{R}$ measures the reconstruction quality by comparing the recovered images with the ground-truth samples $\tilde{\mathcal{D}}_k$. Second, we introduce bottleneck layers to compress the intermediate feature space, thereby reducing the amount of information that can be exploited for data reconstruction and mitigating potential information leakage. We adopt the same configuration as in ResSFL \cite{li2022ressfl}, using VGG-11 with the split layer set at the second layer. We implement the bottleneck layers using a pair of Conv2D layers, with a channel size of $8$ and a stride of $1$. The simulated inversion model is designed as a shallow convolutional autoencoder with a channel size of $16$, while the attacker's inversion model is implemented as a shallow convolutional autoencoder with an internal channel size of $64$. We evaluate image quality using Mean Squared Error (MSE), Structural Similarity Index (SSIM), and Peak Signal-to-Noise Ratio (PSNR), where lower MSE, higher SSIM, and higher PSNR indicate better image reconstruction quality. The experimental results are presented in the Table \ref{privacy_table}. As shown in the table, SCALA-PE demonstrates an effective capability to resist model inversion attacks while maintaining high model accuracy, which becomes more evident on complex datasets. For example, on the CIFAR-100 dataset, the MSE increases by approximately 6.8 times (from 0.0048 to 0.0328), whereas the accuracy decreases by only about two percentage points.

\section{Conclusion}
We proposed SCALA to address the issue of label distribution skew in SFL. We first concatenated activations output by participating clients to serve as the input of server-side model training. Then we proposed loss functions with logit adjustments for the server-side and client-side models. We performed detailed theoretical analysis and extensive experiments to verify the effectiveness of SCALA.

\noindent\textbf{Limitations:} Like other SL and SFL algorithms, SCALA has two main limitations. First, SCALA requires the transmission of labels, which poses a risk of privacy leakage. Incorporating privacy preserving mechanisms of SFL \cite{li2022ressfl,lyu2023optimal} into SCALA to enhance data security is a promising direction for future work. Second, SCALA relies on the frequent exchange of intermediate activations between clients and the server, which can incur substantial communication overhead. Designing activation compression schemes or more communication-efficient variants of SCALA to reduce this overhead while preserving model performance is another valuable direction.

\appendices
\section{Proof of Theorem 1}
\label{proof_theorem1}
We define $\boldsymbol{g}_{s,k}(\cdot)$  as the stochastic gradients of $ F_k(\cdot) $ with respect to $ \mathbf{w}_s $, and $ \boldsymbol{g}_{c,k}(\cdot) $ as the stochastic gradients of $ F_k(\cdot) $ with respect to $ \mathbf{w}_c $. We define $ C $ as the average number of participating clients per global iteration and assume that all clients have an identical participation probability of $ q_k = \frac{C}{K} $. Let $ \mathbf{I}_k^t $ be a binary indicator denoting whether client $ k $ participates in training during iteration $ t $, which follows a Bernoulli distribution with probability $ q_k $. Accordingly, the overall client participation rate is given by $ \rho = q_k = \frac{C}{K} $.
The expected gradients of server-side model is defined as $\nabla_{\mathbf{w}_s} \tilde{F}^{t,i}$, which is unbiasedly estimated by the stochastic gradients as
\allowdisplaybreaks
\begin{align}
    \tilde{\boldsymbol{g}}_s^{t,i} = \sum_{k=1}^K \frac{\mathbf{I}_k^t a_k}{q_k} \boldsymbol{g}_{s,k}^{t,i}([\mathbf{w}_s^{t,i}; \mathbf{w}_{c,k}^{t,i}]),
\end{align}
where $ a_k = \frac{|\mathcal{D}_k|}{\sum_{k\in \mathcal{C}}|\mathcal{D}_k|} $ is the data size proportion. Then, we propose the following lemma:
\begin{lemma}\label{norm_diff}
For the squared norm difference of the server-side model during local iterations, we have:
\begin{align}
    &\mathbb{E}\left[\left\|\mathbf{w}_s^{t,\iota} - \mathbf{w}_s^t\right\|^2\right] \nonumber \\
    &\leq \frac{2 \eta^2 \iota N_s \sigma_{\text{max}}^2}{B} 
+ 2 \eta^2 \iota \sum_{i=0}^{\iota-1} \mathbb{E}\left[\left\|\nabla_{\mathbf{w}_s} \tilde{F}^{t,i}\right\|^2\right]
\end{align}
\end{lemma}
\begin{proof}
\begin{align}
&\mathbb{E}\left[\left\|\mathbf{w}_s^{t,\iota} - \mathbf{w}_s^t\right\|^2\right] 
= \eta^2 \mathbb{E}\left[\left\|\sum_{i=0}^{\iota-1} \tilde{\mathbf{g}}_s^{t,i}\right\|^2\right] \nonumber \\
&\leq \eta^2 \mathbb{E}\left[\left\|\sum_{i=0}^{\iota-1} \left(\tilde{\mathbf{g}}_s^{t,i} - \nabla_{\mathbf{w}_s} \tilde{F}^{t,i} + \nabla_{\mathbf{w}_s} \tilde{F}^{t,i}\right)\right\|^2\right] \nonumber \\
&\overset{(a)}{\leq} 2 \eta^2 \mathbb{E}\left[\left\|\sum_{i=0}^{\iota-1} \left(\tilde{\mathbf{g}}_s^{t,i} - \nabla_{\mathbf{w}_s} \tilde{F}^{t,i}\right)\right\|^2\right] \nonumber \\
& \quad\quad\quad\quad + 2 \eta^2 \iota \sum_{i=0}^{\iota-1} \mathbb{E}\left[\left\|\nabla_{\mathbf{w}_s} \tilde{F}^{t,i}\right\|^2\right] \nonumber \\
&\overset{(b)}{\leq} 2 \eta^2 \sum_{i=0}^{\iota-1} \mathbb{E}\left[\left\|\tilde{\mathbf{g}}_s^{t,i} - \nabla_{\mathbf{w}_s} \tilde{F}^{t,i}\right\|^2\right] \nonumber \\
&\quad\quad\quad\quad + 2 \eta^2 \iota \sum_{i=0}^{\iota-1} \mathbb{E}\left[\left\|\nabla_{\mathbf{w}_s} \tilde{F}^{t,i}\right\|^2\right] \nonumber \\
&\leq \frac{2 \eta^2 \iota N_s \sigma_{\text{max}}^2}{B} 
+ 2 \eta^2 \iota \sum_{i=0}^{\iota-1} \mathbb{E}\left[\left\|\nabla_{\mathbf{w}_s} \tilde{F}^{t,i}\right\|^2\right],
\end{align}
where (a) is derived through Jensen's inequality and (b) is because the inner product is zero due to the unbiasedness of stochastic gradient.
\end{proof}

\par
Under Assumption 1, the decrease of the loss function can be bounded as
\begin{align}\label{eqn:bound_all}
&\mathbb{E}[F(\mathbf{w}^{t+1})] - F(\mathbf{w}^t) \nonumber \\
&\leq \mathbb{E}\left[\langle \nabla_{\mathbf{w}_s} F(\mathbf{w}^t), \mathbf{w}_s^{t+1} - \mathbf{w}_s^t \rangle\right] 
+ \frac{\gamma}{2} \mathbb{E}[\|\mathbf{w}_s^{t+1} - \mathbf{w}_s^t\|^2]  \nonumber \\
&\quad + \mathbb{E}\left[\langle \nabla_{\mathbf{w}_c} F(\mathbf{w}^t), \mathbf{w}_c^{t+1} - \mathbf{w}_c^t \rangle\right] 
+ \frac{\gamma}{2} \mathbb{E}[\|\mathbf{w}_c^{t+1} - \mathbf{w}_c^t\|^2].
\end{align}
For the server, we have 
\begin{align}
    &\mathbb{E}\left[\langle \nabla_{\mathbf{w}_s} F(\mathbf{w}^t), \mathbf{w}_s^{t+1} - \mathbf{w}_s^t \rangle\right] \nonumber \\
    &= -\eta \sum_{i=0}^{I-1} \mathbb{E}\left[\langle \nabla_{\mathbf{w}_s} F(\mathbf{w}^t), \mathbf{\tilde{g}}_s^{t,i} \rangle\right]  \nonumber \\
&= -\eta \sum_{i=0}^{I-1} \mathbb{E}\left[\frac{1}{2} \left\|\nabla_{\mathbf{w}_s} F(\mathbf{w}^t)\right\|^2 
+ \frac{1}{2} \left\|\nabla_{\mathbf{w}_s} \tilde{F}^{t,i}\right\|^2 \right. \nonumber \\
&\quad-\left.\frac{1}{2} \left\|\nabla_{\mathbf{w}_s} F(\mathbf{w}^t) - \nabla_{\mathbf{w}_s} \tilde{F}^{t,i}\right\|^2 \right] \nonumber \\
&\leq -\frac{1}{2} \eta I \left\|\nabla_{\mathbf{w}_s} F(\mathbf{w}^t)\right\|^2 
- \frac{1}{2} \eta \sum_{i=0}^{I-1} \mathbb{E}\left[\left\|\nabla_{\mathbf{w}_s} \tilde{F}^{t,i}\right\|^2\right] \nonumber \\
 &\quad+ \frac{1}{2} \eta \sum_{i=0}^{I-1} \mathbb{E}\left[\left\|\nabla_{\mathbf{w}_s} F(\mathbf{w}^t) - \nabla_{\mathbf{w}_s} \tilde{F}^{t,i}\right\|^2\right]. 
\end{align}
For $\mathbb{E}\left[\left\|\nabla_{\mathbf{w}_s} F(\mathbf{w}^t) - \nabla_{\mathbf{w}_s} \tilde{F}^{t,i}\right\|^2\right] $ we have
\begin{align}
    &\mathbb{E}\left[\left\|\nabla_{\mathbf{w}_s} F(\mathbf{w}^t) - \nabla_{\mathbf{w}_s} \tilde{F}^{t,i}\right\|^2\right]  \nonumber \\
&\leq \mathbb{E}\left[\left\|\sum_{k=1}^K a_k \left(\nabla_{\mathbf{w}_s} F_k(\mathbf{w}^t) - \nabla_{\mathbf{w}_s} F_k\left(\left[\mathbf{w}_s^{t,i}; \mathbf{w}_{c,k}^{t,i}\right]\right)\right)\right\|^2\right]  \nonumber \\
&\leq \sum_{k=1}^K a_k \mathbb{E}\left[\left\|\nabla_{\mathbf{w}_s} F_k(\mathbf{w}^t) - \nabla_{\mathbf{w}_s} F_k^{t,i} \right.\right. \nonumber \\
& \quad\quad\quad\quad \left.\left. + \nabla_{\mathbf{w}_s} F_k^{t,i} - \nabla_{\mathbf{w}_s} F_k\left(\left[\mathbf{w}_s^{t,i}; \mathbf{w}_{c,k}^{t,i}\right]\right)\right\|^2\right]  \nonumber \\
&\leq 2 \sum_{k=1}^K a_k \gamma^2 \mathbb{E}\left[\left\|\mathbf{w}_s^t - \mathbf{w}_s^{t,i}\right\|^2\right] 
+ 2  \nu^2  \nonumber \\
&\leq 2 \gamma^2 \mathbb{E}\left[\left\|\mathbf{w}_s^t - \mathbf{w}_s^{t,i}\right\|^2\right] 
+ 2  \nu^2.
\end{align}
Therefore, we can obtain
\begin{align}\label{eqn:server_bound_1}
\mathbb{E}&\left[\langle \nabla_{\mathbf{w}_s} F(\mathbf{w}^t), \mathbf{w}_s^{t+1} - \mathbf{w}_s^t \rangle\right] + \frac{\gamma}{2} \mathbb{E}\left[\left\|\mathbf{w}_s^{t+1} - \mathbf{w}_s^t\right\|^2\right]\nonumber \\
&\overset{(a)}{\leq} -\frac{1}{2} \eta I \left\|\nabla_{\mathbf{w}_s} F(\mathbf{w}^t)\right\|^2 
+ \eta I \nu^2 \nonumber \\
&\quad + \frac{2 \eta^3 \gamma^2 I^2 N_s \sigma_{\text{max}}^2}{B} 
+ \frac{\gamma \eta^2 I N_s \sigma_{\text{max}}^2}{B} \nonumber \\
&\quad + \left(2 \eta^3 \gamma^2 I^2 + \gamma \eta^2 I - \frac{1}{2} \eta\right) \sum_{i=0}^{I-1} \mathbb{E}\left[\left\|\nabla_{\mathbf{w}_s} \tilde{F}^{t,i}\right\|^2\right], 
\end{align}
where (a) is derived based on Lemma \ref{norm_diff}. Assume that $\eta \leq \frac{1}{8 \gamma I}$, we have $2 \eta^3 \gamma^2 I^2 + \gamma \eta^2 I - \frac{1}{4} \eta \leq 0$, then \eqref{eqn:server_bound_1} can be further bounded as
\begin{align}\label{eqn:server_bound_all}
    \mathbb{E}&\left[\langle \nabla_{\mathbf{w}_s} F(\mathbf{w}^t), \mathbf{w}_s^{t+1} - \mathbf{w}_s^t \rangle\right] + \frac{\gamma}{2} \mathbb{E}\left[\left\|\mathbf{w}_s^{t+1} - \mathbf{w}_s^t\right\|^2\right]\nonumber \\
    &\leq-\frac{1}{4} \eta I \left\|\nabla_{\mathbf{w}_s} F(\mathbf{w}^t)\right\|^2 
+ \eta I \nu^2 + \nonumber \\
&\quad \frac{2 \eta^3 \gamma^2 I^2 N_s \sigma_{\text{max}}^2}{B} + \frac{\gamma \eta^2 I N_s \sigma_{\text{max}}^2}{B}.
\end{align}
For the client, we have
\begin{align}
&\mathbb{E}\left[\left\langle \nabla_{\mathbf{w}_c} F(\mathbf{w}^t), \mathbf{w}_c^{t+1} - \mathbf{w}_c^t \right\rangle\right] \nonumber\\
&\overset{(a)}{\leq} - \frac{1}{2} \eta I \left\|\nabla_{\mathbf{w}_c} F(\mathbf{w}^t)\right\|^2+\nonumber\\
&\frac{1}{2} \eta \sum_{i=0}^{I-1} \sum_{k=1}^K a_k \mathbb{E}\left[\left\|\nabla_{\mathbf{w}_c} F_k (\mathbf{w}^t) - \nabla_{\mathbf{w}_c} F_k \left(\left[\mathbf{w}_s^{t,i}; \mathbf{w}_{c,k}^{t,i}\right]\right)\right\|^2\right] \nonumber \\
&\quad 
- \frac{1}{2} \eta \sum_{i=0}^{I-1} \mathbb{E}\left[\left\|\sum_{k=1}^K a_k \nabla_{\mathbf{w}_c} F_k \left(\left[\mathbf{w}_s^{t,i}; \mathbf{w}_{c,k}^{t,i}\right]\right)\right\|^2\right],
\end{align}
where (a) is derived through Jensen’s inequality. According to \cite{yang2021achieving,wanglightweight}, when $\eta\leq\frac1{\sqrt{30}\gamma I}$, we have
\begin{align}
    \mathbb{E}\left[\left\|\mathbf{w}_c^t - \mathbf{w}_{c,k}^{t,i}\right\|^2\right] 
\leq& \frac{5 \eta^2 I N_c \sigma_{\max}^2}{B_k} 
+ 30 \eta^2 I^2 N_c \kappa_{\max}^2 \nonumber \\
&+ 30 \eta^2 I^2 \left\|\nabla_{\mathbf{w}_c} F(\mathbf{w}_c^t)\right\|^2.
\end{align}
Thus for $\mathbb{E}\left[\left\|\nabla_{\mathbf{w}_c} F_k (\mathbf{w}^t) - \nabla_{\mathbf{w}_c} F_k \left(\left[\mathbf{w}_s^{t,i}; \mathbf{w}_{c,k}^{t,i}\right]\right)\right\|^2\right]$, we have
\begin{align}
\mathbb{E}&\left[\left\|\nabla_{\mathbf{w}_c} F_k (\mathbf{w}^t) - \nabla_{\mathbf{w}_c} F_k \left(\left[\mathbf{w}_s^{t,i}; \mathbf{w}_{c,k}^{t,i}\right]\right)\right\|^2\right] \nonumber\\
&\overset{(a)}{\leq} \gamma^2 \mathbb{E}\left[\left\|\mathbf{w}_c^t - \mathbf{w}_{c,k}^{t,i}\right\|^2\right]. \nonumber \\
&\leq \frac{5 \eta^2 \gamma^2 I N_c \sigma_{\max}^2}{B_k} 
+ 30 \eta^2 \gamma^2 I^2 N_c \kappa_{\max}^2 \nonumber \\
&\quad+ 30 \eta^2 \gamma^2 I^2 \left\|\nabla_{\mathbf{w}_c} F(\mathbf{w}^t_c)\right\|^2,
\end{align}
where (a) is derived based on Assumption 1. Therefore, we can obtain
\begin{align}
    \mathbb{E}&\left[\langle \nabla_{\mathbf{w}_c} F(\mathbf{w}^t), \mathbf{w}_c^{t+1} - \mathbf{w}_c^t \rangle\right]  \nonumber \\
&\leq \frac{5 \eta^3 \gamma^2 I^2 N_c \sigma_{\max}^2 K}{2 B} + 15 \eta^3 \gamma^2 I^3 N_c \kappa_{\max}^2 
\nonumber\\
&\quad + \left(15 \eta^3 \gamma^2 I^3 - \frac{1}{2} \eta I\right) \left\|\nabla_{\mathbf{w}_c} F(\mathbf{w}^t)\right\|^2 \nonumber\\
&\quad - \frac{1}{2} \eta \sum_{i=0}^{I-1} \mathbb{E}\left[\left\|\sum_{k=1}^K a_k \nabla_{\mathbf{w}_c} F_k\left(\left[\mathbf{w}_s^{t,i}; \mathbf{w}_{c,k}^{t,i}\right]\right)\right\|^2\right].
\end{align}
Then for $\frac{\gamma}{2} \mathbb{E}\left[\left\|\mathbf{w}_c^{t+1} - \mathbf{w}_c^t\right\|^2\right]$, we have
\begin{align}\label{eqn:client_model_tmp}
    &\frac{\gamma}{2} \mathbb{E}\left[\left\|\mathbf{w}_c^{t+1} - \mathbf{w}_c^t\right\|^2\right] \nonumber\\
&= \frac{\gamma \eta^2}{2} \mathbb{E}\left[\left\|\sum_{i=0}^{I-1} \sum_{k=1}^K \frac{\mathbf{I}_k^t}{q_k} a_k \mathbf{g}_{c,k}\left(\left[\mathbf{w}_s^{t,i}; \mathbf{w}_{c,k}^{t,i}\right]\right)\right\|^2\right]  \nonumber \\
&\leq \gamma \eta^2 \mathbb{E}\left[\left\|\sum_{i=0}^{I-1} \sum_{k=1}^K \frac{\mathbf{I}_k^t}{q_k} a_k \left(\mathbf{g}_{c,k}\left(\left[\mathbf{w}_s^{t,i}; \mathbf{w}_{c,k}^{t,i}\right]\right) \right.\right.\right. \nonumber\\
& \quad\quad\quad\quad \left.\left.\left.- \nabla_{\mathbf{w}_c} F_k\left(\left[\mathbf{w}_s^{t,i}; \mathbf{w}_{c,k}^{t,i}\right]\right)\right)\right\|^2\right] \nonumber \\ 
&\quad + \gamma \eta^2 \mathbb{E}\left[\left\|\sum_{i=0}^{I-1} \sum_{k=1}^K \frac{\mathbf{I}_k^t}{q_k} a_k \nabla_{\mathbf{w}_c} F_k\left(\left[\mathbf{w}_s^{t,i}; \mathbf{w}_{c,k}^{t,i}\right]\right)\right\|^2\right] \nonumber \\
&\overset{(a)}{\leq} \gamma \eta^2 I \sum_{k=1}^K \frac{a_k^2 N_c \sigma_{\max}^2}{q_k B_k} 
+ \nonumber\\
&\gamma \eta^2 \sum_{k=1}^K \mathbb{E}\left[\left\|\frac{a_k}{q_k} \left(\mathbf{I}_k^t - q_k\right) \sum_{i=0}^{I-1} 
\nabla_{\mathbf{w}_c} F_k\left(\left[\mathbf{w}_s^{t,i}; \mathbf{w}_{c,k}^{t,i}\right]\right)\right\|^2\right] \nonumber \\ 
&\quad + \gamma \eta^2 \mathbb{E}\left[\left\|\sum_{i=0}^{I-1} \sum_{k=1}^K a_k \nabla_{\mathbf{w}_c} F_k\left(\left[\mathbf{w}_s^{t,i}; \mathbf{w}_{c,k}^{t,i}\right]\right)\right\|^2\right],
\end{align}
where (a) is because the inner product is zero due to the unbiasedness of stochastic gradient. Since $\mathbf{I}_k^t$ Follows the Bernoulli distribution, we have  $\mathbb{E}[\mathbf{I}^t_k] = q_k$ and $\mathrm{Var}[\mathbf{I}^t_k] = q_k(1 - q_k)$. Thus \eqref{eqn:client_model_tmp} can be bounded as 
\begin{align}\label{eqn:client_model}
& \frac{\gamma}{2} \mathbb{E}\left[\left\|\mathbf{w}_c^{t+1} - \mathbf{w}_c^t\right\|^2\right]
\leq \frac{\gamma \eta^2 I N_c \sigma_{\max}^2}{B \rho} \nonumber\\
&\quad + \frac{\gamma \eta^2 I}{\rho} \sum_{k=1}^K a_k^2 \sum_{i=0}^{I-1} \mathbb{E}\left[\left\|\nabla_{\mathbf{w}_c} F_k\left(\left[\mathbf{w}_s^{t,i}; \mathbf{w}_{c,k}^{t,i}\right]\right)\right\|^2\right]  \nonumber \\
&\quad + \gamma \eta^2 I \sum_{i=0}^{I-1} \mathbb{E}\left[\left\|\sum_{k=1}^K a_k \nabla_{\mathbf{w}_c} F_k\left(\left[\mathbf{w}_s^{t,i}; \mathbf{w}_{c,k}^{t,i}\right]\right)\right\|^2\right].
\end{align}
For $\mathbb{E}\left[\left\|\nabla_{\mathbf{w}_c} F_k\left(\left[\mathbf{w}_s^{t,i}; \mathbf{w}_{c,k}^{t,i}\right]\right)\right\|^2\right] $, we have
\begin{align}\label{eqn:client_fk}
    \mathbb{E}&\left[\left\|\nabla_{\mathbf{w}_c} F_k\left(\left[\mathbf{w}_s^{t,i}; \mathbf{w}_{c,k}^{t,i}\right]\right)\right\|^2\right]  \nonumber \\
&\leq 3 \gamma^2 \mathbb{E}\left[\left\|\mathbf{w}_c^t - \mathbf{w}_{c,k}^{t,i}\right\|^2\right] 
+ 3 N_c \kappa_{max}^2 
+ 3 \left\|\nabla_{\mathbf{w}_c} F(\mathbf{w}^t)\right\|^2 \nonumber \\
&\leq \frac{15 \eta^2 \gamma^2 I N_c \sigma_{max}^2}{B_k} 
+ 90 \eta^2 \gamma^2 I^2 N_c \kappa_{max}^2 
+ 3 N_c \kappa_{max}^2 \nonumber \\
&\quad+ 3 \left(30 \eta^2 \gamma^2 I^2 + 1\right) \left\|\nabla_{\mathbf{w}_c} F(\mathbf{w}^t)\right\|^2.
\end{align}
Incorporating \eqref{eqn:client_fk} into \eqref{eqn:client_model}, we can further bound $\frac{\gamma}{2} \mathbb{E}\left[\left\|\mathbf{w}_c^{t+1} - \mathbf{w}_c^t\right\|^2\right]$ as
\begin{align}
    &\frac{\gamma}{2} \mathbb{E}\left[\left\|\mathbf{w}_c^{t+1} - \mathbf{w}_c^t\right\|^2\right] \nonumber \\
&\leq \frac{\gamma \eta^2 I N_c \sigma_{max}^2}{B \rho} 
+ \frac{15 \eta^4 \gamma^3 I^3 N_c \sigma_{max}^2}{B \rho} \nonumber \\
&\quad+ \frac{90 \eta^4 \gamma^3 I^4 N_c \kappa_{max}^2}{\rho} 
+ \frac{3 \gamma \eta^2 I^2 N_c \kappa_{max}^2}{\rho} \nonumber \\
& \quad+ \frac{3 \gamma \eta^2 I^2 \left(30 \eta^2 \gamma^2 I^2 + 1\right) \left\|\nabla_{\mathbf{w}_c} F(\mathbf{w}^t)\right\|^2}{\rho} \nonumber \\
&\quad + \gamma \eta^2 I \sum_{i=0}^{I-1} \mathbb{E}\left[\left\|\sum_{k=1}^K a_k \nabla_{\mathbf{w}_c} F_k\left(\left[\mathbf{w}_s^{t,i}; \mathbf{w}_{c,k}^{t,i}\right]\right)\right\|^2\right].
\end{align}
Therefore, we can obtain
\begin{align}\label{eqn:client_bound_1}
    \mathbb{E}&\left[\langle \nabla_{\mathbf{w}_c} F(\mathbf{w}^t), \mathbf{w}_c^{t+1} - \mathbf{w}_c^t \rangle\right] 
+ \frac{\gamma}{2} \mathbb{E}[\|\mathbf{w}_c^{t+1} - \mathbf{w}_c^t\|^2]\nonumber \\
&\leq \frac{5 \eta^3 \gamma^2 I^2 N_c \sigma_{\max}^2 K}{2 B} + 15 \eta^3 \gamma^2 I^3 N_c \kappa_{\max}^2 \nonumber \\
&\quad+ \left(15 \eta^3 \gamma^2 I^3 - \frac{1}{2} \eta I\right) \left\|\nabla_{\mathbf{w}_c} F(\mathbf{w}^t)\right\|^2 \nonumber\\
&\quad+ \frac{\gamma \eta^2 I N_c \sigma_{max}^2}{B \rho} 
+ \frac{15 \eta^4 \gamma^3 I^3 N_c \sigma_{max}^2}{B \rho} \nonumber \\
&\quad + \frac{90 \eta^4 \gamma^3 I^4 N_c \kappa_{max}^2}{\rho} 
+ \frac{3 \gamma \eta^2 I^2 N_c \kappa_{max}^2}{\rho} \nonumber \\
&\quad + \frac{3 \gamma \eta^2 I^2 \left(30 \eta^2 \gamma^2 I^2 + 1\right) \left\|\nabla_{\mathbf{w}_c} F(\mathbf{w}^t)\right\|^2}{\rho} \nonumber \\
&+(\gamma \eta^2 I-\frac{\eta}{2}) \sum_{i=0}^{I-1} \mathbb{E}\left[\left\|\sum_{k=1}^K a_k \nabla_{\mathbf{w}_c} F_k\left(\left[\mathbf{w}_s^{t,i}; \mathbf{w}_{c,k}^{t,i}\right]\right)\right\|^2\right].
\end{align}
Assume that $\eta \leq \frac{\rho}{36 \gamma I}$, we have $\gamma \eta^2 I-\frac{1}{2}\eta$ and $\frac{3 \gamma \eta^2 I^2 \left(30 \eta^2 \gamma^2 I^2 + 1\right)}{\rho} + 15 \eta^3 \gamma^2 I^3 - \frac{1}{2} \eta I \leq \frac{1}{4} \eta I$, then \eqref{eqn:client_bound_1} can be further bounded as 
\begin{align}\label{eqn:client_bound_all}
    &\mathbb{E}\left[\langle \nabla_{\mathbf{w}_c} F(\mathbf{w}^t), \mathbf{w}_c^{t+1} - \mathbf{w}_c^t \rangle\right] 
+ \frac{\gamma}{2} \mathbb{E}[\|\mathbf{w}_c^{t+1} - \mathbf{w}_c^t\|^2]\nonumber \\
&\leq \frac{1}{4} \eta I \left\|\nabla_{\mathbf{w}_c} F(\mathbf{w}^t)\right\|^2 + \frac{5 \eta^3 \gamma^2 I^2 N_c \sigma_{\max}^2 K}{2 B} \nonumber\\
&+ 15 \eta^3 \gamma^2 I^3 N_c \kappa_{\max}^2 + \frac{\gamma \eta^2 I N_c \sigma_{max}^2}{B \rho} + \frac{15 \eta^4 \gamma^3 I^3 N_c \sigma_{max}^2}{B \rho}    \nonumber\\
&
+ \frac{90 \eta^4 \gamma^3 I^4 N_c \kappa_{max}^2}{\rho} 
+ \frac{3 \gamma \eta^2 I^2 N_c \kappa_{max}^2}{\rho}.
\end{align}
Incorporating \eqref{eqn:client_bound_all} and \eqref{eqn:server_bound_all} into \eqref{eqn:bound_all}, and then taking the total expectation and averaging over all rounds, we can obtain the convergence rate:
\begin{align}
    &\frac{1}{T} \sum_{t=0}^{T-1} \mathbb{E}\left[\left\|\nabla_{\mathbf{w}} F(\mathbf{w}^t)\right\|^2\right] 
\leq \frac{F(\mathbf{w}^0) - F^*}{\frac{1}{4} \eta T I} 
+ 4 \nu^2 \nonumber\\
&+ \frac{8 \eta^2 \gamma^2 I N_s \sigma_{\text{max}}^2}{B} 
+ \frac{4 \gamma \eta N_s \sigma_{\text{max}}^2}{B}+ \frac{20 \eta^2 \gamma^2 I N_c \sigma_{max}^2 K}{2B}  \nonumber \\
&
+ 60 \eta^2 \gamma^2 I^2 N_c \kappa_{max}^2 
+ \frac{360 \eta^3 \gamma^3 I^3 N_c \kappa_{max}^2}{\rho} \nonumber \\
&+ \frac{12 \gamma \eta I N_c \kappa_{max}^2}{\rho} + \frac{4 \gamma \eta N_c \sigma_{max}^2}{B \rho} 
+ \frac{60 \eta^3 \gamma^3 I^2 N_c \sigma_{max}^2}{B \rho}.
\end{align}
Assume that $\eta = \Theta\left(\frac{1}{\sqrt{T I}}\right)$ and omit the lower-order terms, we can obtain the convergence of the concatenated activations enabled SFL:
\begin{align}
&\frac{1}{T} \sum_{t=0}^{T-1} \mathbb{E}\left[\left\|\nabla_{\mathbf{w}} F(\mathbf{w}^t)\right\|^2\right] \leq \mathcal{O} \left( \frac{F(\mathbf{w}^0) - F^*}{\sqrt{T I}}+\nu^2 \right) \nonumber \\
&\quad+ \mathcal{O} \left( \frac{N_s \sigma_{\max}^2}{B \sqrt{T I}} \right) + \mathcal{O} \left( \frac{N_c}{\rho\sqrt{T}}\left(\frac{ \sigma_{\max}^2}{B\sqrt{I}} + \sqrt{I} \kappa_{\max}^2\right) \right),
\end{align}
which completes the proof.
\section{Proof of Theorem 2}
\label{proof_theorem2}
We first analyze the update process of a classifier when using the softmax cross-entropy loss function, as illustrated in the following lemma:
\begin{lemma}\label{lemma1}
    Consider a dataset $\mathcal{D}$ with $M$ classes and assume that Assumption 5 holds, we can obtain the update of the logit when using the softmax cross-entropy loss function as
    \begin{align}\label{eqn:lemma1}
        \Delta \zeta_y\cdot \pi_y =\eta P(y)\text{avg}_y\left(\frac{\sum_{y^{\prime}\neq y}e^{s_{y^{\prime}}(\mathbf{x})-s_{y}(\mathbf{x})} }{1+\sum_{y^{\prime}\neq y}e^{s_{y^{\prime}}(\mathbf{x})-s_{y}(\mathbf{x})} }\right)\pi_y\cdot\pi_y,
    \end{align}
where $\text{avg}_y(\cdot)=\frac{1}{|\mathcal{D}_y|}\sum_{\mathbf{x}\in\mathcal{D}_y}(\cdot)$. 
\end{lemma}
\begin{proof}
For simplicity of presentation, we denote $S(y^{\prime})$ as $s_{y^{\prime}}(\mathbf{x}) - s_{y}(\mathbf{x})$. Given a dataset $\mathcal{D}$ with $M$ classes, we can obtain the update of the classifier of label $y$ when using the softmax cross-entropy loss function as
\begin{align}\label{eqn:Delta1}
    &\Delta \zeta_y =  -\eta P(y) \frac{1}{|\mathcal{D}_y|}\sum_{\mathbf{x}_i\in\mathcal{D}_y}\frac{\partial g(y, s(\mathbf{x}_i))}{\partial \zeta_y} 
    - \nonumber\\
    &\quad\quad-\eta \sum_{y^{\prime}\neq y}P(y^{\prime})\frac{1}{|\mathcal{D}_{y^{\prime}}|}\sum_{\mathbf{x}_i\in\mathcal{D}_{y^{\prime}}}\frac{\partial g(y^{\prime}, s(\mathbf{x}_i))}{\partial \zeta_y} \nonumber\\
     &= \eta P(y) \frac{1}{|\mathcal{D}_y|}\sum_{\mathbf{x}_i\in\mathcal{D}_y}(1- p_y(\mathbf{x}_i))\pi(\mathbf{x}_i)\nonumber\\
    &\quad\quad
     - \eta \sum_{y^{\prime}\neq y}P(y^{\prime})\frac{1}{|\mathcal{D}_{y^{\prime}}|}\sum_{\mathbf{x}_i\in\mathcal{D}_{y^{\prime}}}p_y(\mathbf{x}_i)\pi(\mathbf{x}_i)\nonumber\\
     &=\eta P(y)\text{avg}_y\left(\frac{\sum_{y^{\prime}\neq y}e^{S(y^{\prime})} }{1+\sum_{y^{\prime}\neq y}e^{S(y^{\prime})} }\pi(\mathbf{x})\right) \nonumber\\
    &\quad\quad- \eta \sum_{y^{\prime}\neq y}P(y^{\prime})\text{avg}_{y^{\prime}}\left(\frac{1}{1+\sum_{y^{\prime}\neq y}e^{S(y^{\prime})} }\pi(\mathbf{x})\right),
\end{align}
where $\mathcal{D}_y$ is the subset of dataset $\mathcal{D}$ with label $y$ and $\text{avg}_y(\cdot)=\frac{1}{|\mathcal{D}_y|}\sum_{\mathbf{x}\in\mathcal{D}_y}(\cdot)$. 
By introducing \eqref{eqn:Delta1}, we can obtain the update of the logit of label $y$ as
\begin{align}\label{eqn:Delta2}
    &\Delta \zeta_y\cdot \pi_y
    =\eta P(y)\text{avg}_y\left(\frac{\sum_{y^{\prime}\neq y}e^{S(y^{\prime})} }{1+\sum_{y^{\prime}\neq y}e^{S(y^{\prime})} }\pi(\mathbf{x})\cdot \pi_y\right) \nonumber\\
    &\quad- \eta \sum_{y^{\prime}\neq y}P(y^{\prime})\text{avg}_{y^{\prime}}\left(\frac{1}{1+\sum_{y^{\prime}\neq y}e^{S(y^{\prime})} }\pi(\mathbf{x})\cdot \pi_y\right),
\end{align}
where $\pi_y$ is the averaged model feature of label $y$ defined as $\pi_y=\frac{1}{|\mathcal{D}_y|}\sum_{\mathbf{x}_i\in\mathcal{D}_y}\pi(\mathbf{x}_i)$. When model feature $\pi(\mathbf{x})$ of label $y$ is similar, according to Property 1 in \cite{zhang2022federated}, \eqref{eqn:Delta2} can be approximated as
\begin{align}\label{eqn:lemma1_ori}
    &\Delta \zeta_y\cdot \pi_y\approx\eta P(y)\text{avg}_y\left(\frac{\sum_{y^{\prime}\neq y}e^{S(y^{\prime})} }{1+\sum_{y^{\prime}\neq y}e^{S(y^{\prime})} }\right)\pi_y\cdot\pi_y \nonumber\\
    &- \eta \sum_{y^{\prime}\neq y}P(y^{\prime})\text{avg}_{y^{\prime}}\left(\frac{1}{1+\sum_{y^{\prime}\neq y}e^{S(y^{\prime})} }\right)\pi_y\cdot\pi_{y^{\prime}}.
\end{align}
When Assumption 5 holds, we can rewrite \eqref{eqn:lemma1_ori} as \eqref{eqn:lemma1} and completes the proof.
\end{proof}
Lemma \ref{lemma1} indicates that the update of the logit $\Delta \zeta_y\cdot \pi_y$ is positively correlated with the label distribution $P(y)$, where $\Delta \zeta_y\cdot \pi_y$ decreases with the reduction of $P(y)$ and ultimately tends to $0$. Therefore, the classifier will exhibit a bias, because it neglects the prediction of low-frequency labels and outputs higher accuracy for high-frequency labels.
\par
Based on the softmax cross-entropy loss function with logit adjustment defined as $-\log\left[\frac{e^{s_y(\mathbf{x})+\log P(y)}}{\sum_{y^{\prime}=1}^M e^{s_{y^{\prime}}(\mathbf{x})+\log P(y^{\prime})}}\right]$, we propose the following lemma:
\begin{lemma}\label{lemma2}
     Let $\Delta \zeta^{\text{bal}} $ be the update of the classifier $\zeta$ when using softmax cross-entropy loss function with logit adjustment.
    Consider a dataset $\mathcal{D}$ with $M$ classes and assume that Assumption 5 holds,  we can obtain the corresponding update of the logit as
    \begin{align}\label{eqn:lemma2}
        &\Delta\zeta^{\text{bal}}_y\cdot\pi_y \nonumber\\
        &=\eta \text{avg}_y\left(\frac{P(y)\sum_{y^{\prime}\neq y}P(y^{\prime})e^{s_{y^{\prime}}(\mathbf{x})-s_y(\mathbf{x})}}{P(y)+\sum_{y^{\prime}\neq y}P(y^{\prime})e^{s_{y^{\prime}}(\mathbf{x})-s_y(\mathbf{x})}}\right)\pi_y\cdot\pi_y.
    \end{align}
\end{lemma}
\begin{proof}
Given a dataset $\mathcal{D}$ with $M$ classes, we can obtain the update of the classifier of label $y$ when using the softmax cross-entropy loss function with logit adjustment as
\begin{align}\label{eqn:Delta3}
    &\Delta \zeta_y =  -\eta P(y) \frac{1}{|\mathcal{D}_y|}\sum_{\mathbf{x}_i\in\mathcal{D}_y}\frac{\partial g^{\text{bal}}(y, s(\mathbf{x}_i))}{\partial \zeta_y} \nonumber\\
    &\quad- \eta \sum_{y^{\prime}\neq y}P(y^{\prime})\frac{1}{|\mathcal{D}_{y^{\prime}}|}\sum_{\mathbf{x}_i\in\mathcal{D}_{y^{\prime}}}\frac{\partial g^{\text{bal}}(y^{\prime}, s(\mathbf{x}_i))}{\partial \zeta_y} \nonumber\\
     &=\eta \text{avg}_y\left(\frac{P(y)\sum_{y^{\prime}\neq y}P(y^{\prime})e^{S(y^{\prime})} }{P(y)+\sum_{y^{\prime}\neq y}P(y^{\prime})e^{S(y^{\prime})} }\pi(\mathbf{x})\right) \nonumber\\
    &\quad- \eta \sum_{y^{\prime}\neq y}P(y^{\prime})\text{avg}_{y^{\prime}}
    \left(\frac{P(y)}{P(y)+\sum_{y^{\prime}\neq y}P(y^{\prime})e^{S(y^{\prime})} }\pi(\mathbf{x})\right).
\end{align}
By introducing \eqref{eqn:Delta3}, we can obtain the update of the logit of label $y$ as
\begin{align}\label{eqn:Delta4}
&\Delta\zeta^{\text{bal}}_y\cdot\pi_y=\eta \text{avg}_y\left(\frac{P(y)\sum_{y^{\prime}\neq y}P(y^{\prime})e^{S(y^{\prime})}}{P(y)+\sum_{y^{\prime}\neq y}P(y^{\prime})e^{S(y^{\prime})}}\pi(\mathbf{x})\cdot\pi_y\right)
\nonumber\\
&- \eta \sum_{y^{\prime}\neq y}P(y^{\prime})\nonumber\\
&\quad\quad\quad \text{avg}_{y^{\prime}}\left(\frac{P(y)}{P(y)+\sum_{y^{\prime}\neq y}P(y^{\prime})e^{S(y^{\prime})} }\pi(\mathbf{x})\cdot\pi_y\right). 
\end{align}
When model feature $\pi(\mathbf{x})$ of label $y$ is similar, according to Property 1 in \cite{zhang2022federated}, \eqref{eqn:Delta4} can be approximated as
\begin{align}\label{eqn:lemma2_ori}
&\Delta\zeta^{\text{bal}}_y\cdot\pi_y\approx\eta \text{avg}_y\left(\frac{P(y)\sum_{y^{\prime}\neq y}P(y^{\prime})e^{S(y^{\prime})}}{P(y)+\sum_{y^{\prime}\neq y}P(y^{\prime})e^{S(y^{\prime})}}\right)\pi_y\cdot\pi_y
\nonumber\\
&- \eta \sum_{y^{\prime}\neq y}P(y^{\prime})\text{avg}_{y^{\prime}}\left(\frac{P(y)}{P(y)+\sum_{y^{\prime}\neq y}P(y^{\prime})e^{S(y^{\prime})} }\right)\pi_y\cdot\pi_{y^{\prime}}. 
\end{align}
When Assumption 5 holds, we can rewrite \eqref{eqn:lemma2_ori} as \eqref{eqn:lemma2} and completes the proof.
\end{proof}


Lemma \ref{lemma2} reveals that, compared with the traditional loss function, a loss function with logit adjustment can balance the updates of the classifier under skewed label distributions and can achieve similar recognition capabilities for both low-frequency and high-frequency labels. Specifically, when $P(y)$ is low, the update of the logit starts from $0$ and increases with the increase of $P(y)$ to promote recognition of low-frequency labels. When $P(y)$ is high, the update of the logit decreases and approaches $0$ with the increase of $P(y)$ to prevent biased updating of high-frequency labels. 
\par
Based on Lemma \ref{lemma1} and Lemma \ref{lemma2}, it is straightforward to derive that as P(y) approaches 1, 
\begin{align}
    \Delta\zeta^{\text{bal}}_y\cdot\pi_y=0,
\end{align}
while  
\begin{align}
    \Delta \zeta_y\cdot \pi_y =\eta \text{avg}_y\left(\frac{\sum_{y^{\prime}\neq y}e^{s_{y^{\prime}}(\mathbf{x})-s_{y}(\mathbf{x})} }{1+\sum_{y^{\prime}\neq y}e^{s_{y^{\prime}}(\mathbf{x})-s_{y}(\mathbf{x})} }\right)\pi_y\cdot\pi_y > 0.
\end{align}
Hence, it can be concluded that
\begin{align}
   \lim_{P(y)\rightarrow 1}  \Delta\zeta^{\text{bal}}y\cdot\pi_y < \lim_{P(y)\rightarrow 1} \Delta\zeta_y\cdot\pi_y.
\end{align}
To clarify the conclusion, we consider a scenario where all labels are uniformly distributed except for label $y$, that is, for all $y^{\prime}\neq y$, $P(y^{\prime}) = \frac{1-P(y)}{M-1}$. Then \eqref{eqn:lemma2} can be rewritten as 
\begin{align}
    &\Delta\zeta^{\text{bal}}_y\cdot\pi_y= \nonumber\\
    &\eta \text{avg}_y\left(\frac{P(y)(1-P(y))\sum_{y^{\prime}\neq y}e^{S(y^{\prime})}}{P(y)(M-1)+(1-P(y))\sum_{y^{\prime}\neq y}e^{S(y^{\prime})}}\right)\pi_y\cdot\pi_y.
\end{align}
Let $E$ denote $\sum_{y^{\prime}\neq y}e^{s_{y^{\prime}}(\mathbf{x})-s_y(\mathbf{x})}$, we have
\begin{align}\label{eqn:par1}
    &\lim_{P(y)\rightarrow 0}\frac{\partial \Delta\zeta^{\text{bal}}_y\cdot\pi_y}{\partial P(y)} = \eta \pi_y\cdot\pi_y,
\end{align}
and
\begin{align}\label{eqn:par2}
    \lim_{P(y)\rightarrow 0}\frac{\partial \Delta\zeta_y\cdot\pi_y}{\partial P(y)} = \eta \text{avg}_y\left(\frac{E}{1+E}\right)\pi_y\cdot\pi_y.
\end{align}
According to \eqref{eqn:par1} and \eqref{eqn:par2}, we have
\begin{align}
    \lim_{P(y)\rightarrow 0}\frac{\partial \Delta\zeta^{\text{bal}}_y\cdot\pi_y}{\partial P(y)} > \lim_{P(y)\rightarrow 0}\frac{\partial \Delta\zeta_y\cdot\pi_y}{\partial P(y)}.
\end{align}
Consequently, as $P(y)$ approaches $0$, $\Delta\zeta^{\text{bal}}_y\cdot\pi_y$ will increase at a faster rate with the increase of $P(y)$. Given that $\Delta\zeta^{\text{bal}}_y\cdot\pi_y =\Delta\zeta_y\cdot\pi_y=0$ when $P(y)=0$, we can draw the conclusion:
\begin{align}
        \lim_{P(y)\rightarrow 0} \Delta\zeta^{\text{bal}}_y\cdot\pi_y > \lim_{P(y)\rightarrow 0} \Delta\zeta_y\cdot\pi_y,
\end{align}
which completes the proof.

\bibliographystyle{IEEEtran}
\bibliography{IEEEabrv,link}

\begin{IEEEbiography}
[{\includegraphics[width=1in,height=1.25in,clip,keepaspectratio]{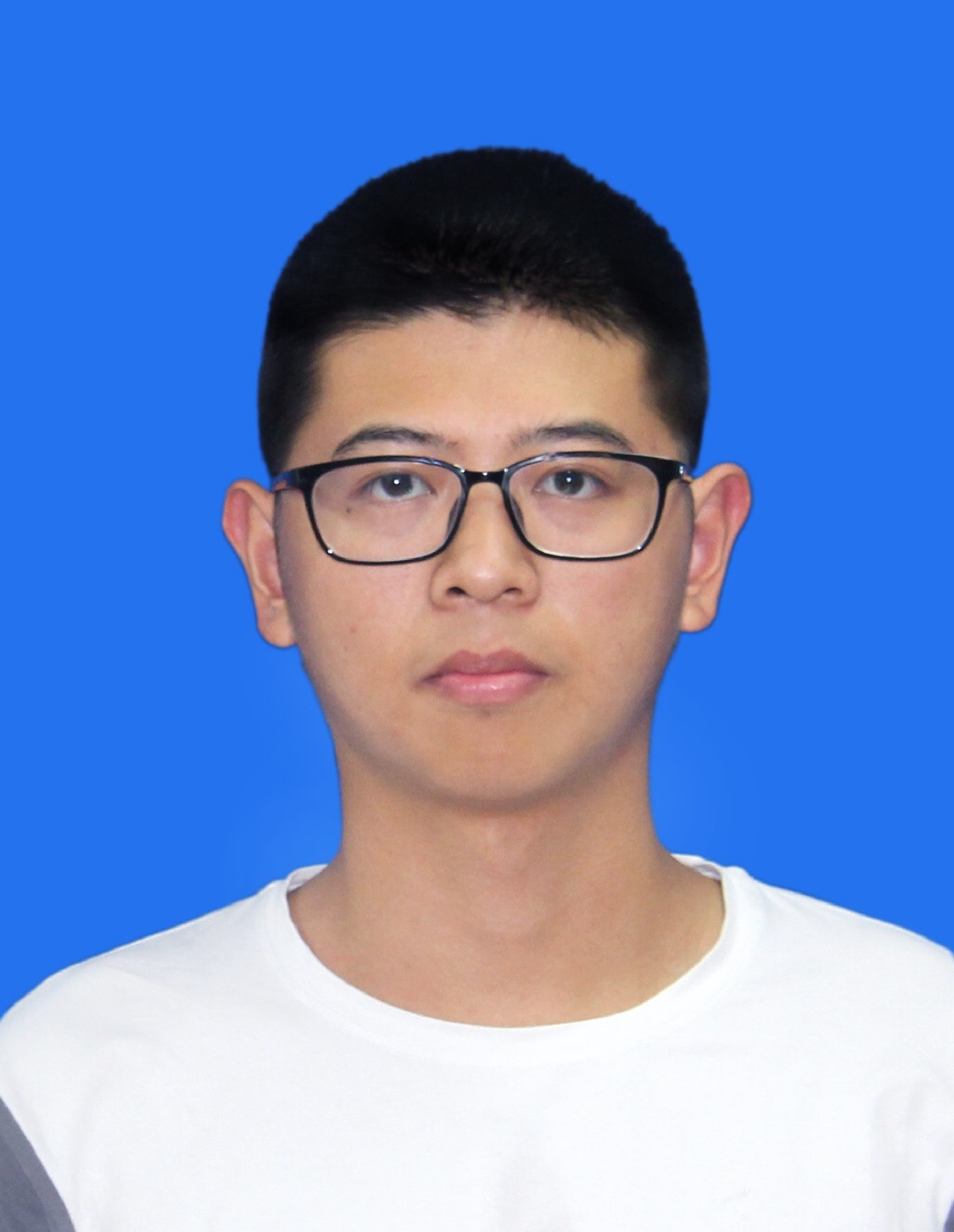}}]
{Jiarong Yang} received the B.S. degree from South China University of Technology, Guangzhou, China, in 2021. He is currently pursuing the Ph.D. degree with the School of Electronic and Information Engineering, South China University of Technology, Guangzhou, China. His research interests include federated learning, Bayesian learning, and machine learning techniques in wireless communications.
\end{IEEEbiography}
	
\begin{IEEEbiography} [{\includegraphics[width=1in,height=1.25in,clip,keepaspectratio]{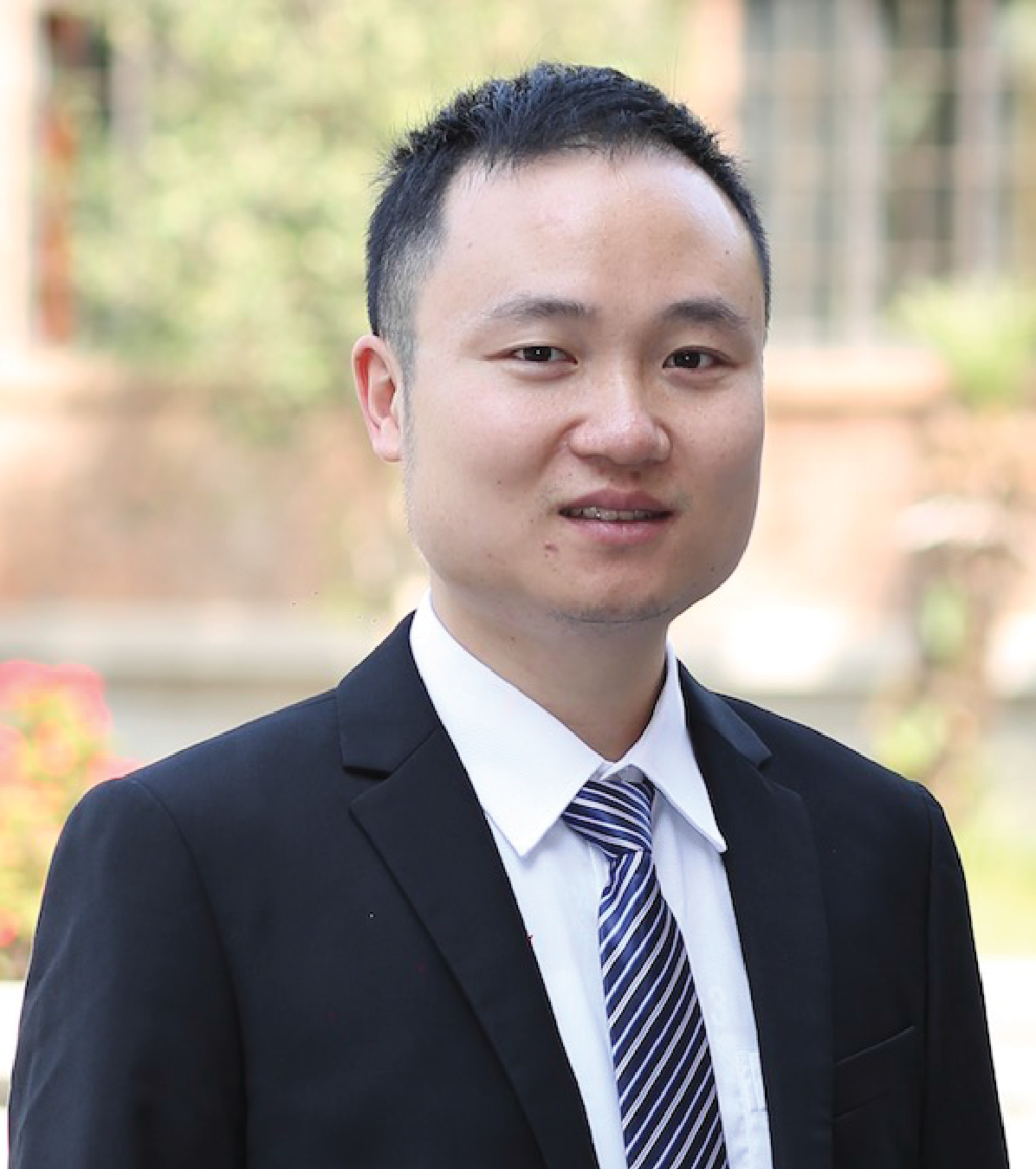}}]
{Yuan Liu}(Senior Member, IEEE) received the Ph.D. degree from Shanghai Jiao Tong University in electronic engineering, China, in 2013. Since 2013, he has been with the School of Electronic and Information Engineering, South China University of Technology, Guangzhou, where he is currently a professor.
He served as an editor for the \textsc{IEEE Communications Letters} and the \textsc{IEEE Access}. His research interests include machine learning, large language models, and edge intelligence. 
\end{IEEEbiography}

\vfill

\end{document}